\documentclass[twoside,11pt]{article}
\usepackage{amsmath,amssymb}
\usepackage{blindtext}
\usepackage{subfigure,placeins}
\usepackage{float}
\usepackage{multirow}
\usepackage{booktabs}
\usepackage{xcolor}

\newcommand{\wstar}{w^*(\lambda)}
\newcommand{\wnhat}{\hat{w}_n(\lambda)}
\newtheorem{assumption}{Assumption}
\usepackage{jmlr2e}

\usepackage{amsmath}

\usepackage{lastpage}

\ShortHeadings{Role of Sample Splitting towards Optimal NNs}{Gong and Zhang}
\firstpageno{1}

\begin{document}

\title{Towards Optimal Neural Networks: the Role of Sample Splitting in Hyperparameter Selection}

\author{\name Shijin Gong \email gongshijin0409@mail.ustc.edu.cn \\
       \addr School of Management\\
       University of Science and Technology of China\\
       Hefei, Anhui, China
       \AND
       \name Xinyu Zhang \email xinyu@amss.ac.cn \\
       \addr Academy of Mathematics and Systems Science\\
       Chinese Academy of Sciences\\
       Beijing, China}

\editor{}

\maketitle



\begin{abstract} 
When artificial neural networks have demonstrated exceptional practical success in a variety of domains, investigations into their theoretical characteristics, such as their approximation power, statistical properties, and generalization performance, have concurrently made significant strides. In this paper, we construct a novel theory for understanding the effectiveness of neural networks, which offers a perspective distinct from prior research. Specifically, we explore the rationale underlying a common practice during the construction of neural network models: sample splitting. Our findings indicate that the optimal hyperparameters derived from sample splitting can enable a neural network model that asymptotically minimizes the prediction risk. We conduct extensive experiments across different application scenarios and network architectures, and the results manifest our theory's effectiveness. 
\end{abstract}




\section{Introduction} \label{sec:1}
Artificial neural networks (NNs) have garnered remarkable success across various industries, leading to a constant stream of research efforts aimed at exploring novel neural network models to enhance the scope of application and improve performance. Over the past few decades, researchers have developed convolutional neural networks (CNNs) \citep{fukushima1980neocognitron,krizhevsky2012imagenet}, recurrent neural networks (RNNs) \citep{rumelhart1986learning,hochreiter1997long}, and further models based on more complicated structures or components \citep[e.g.,][]{scarselli2008graph,vaswani2017attention}, which have demonstrated efficacy in solving a board range of practical problems. In parallel, efforts by scientists to uncover the theoretical characteristics underpinning neural networks, such as approximation power, statistical properties, and generalization ability, have resulted in a tremendous amount of studies and discussions. In what follows, we provide a concise overview of these major strides in neural network research that are currently garnering interest. Within this context, our framework emerges as a novel perspective aiming to expand the ongoing effort.


\citet{hornik1989multilayer} presented an universal approximation theorem for feed-forward neural networks (FNNs), leading to numerous subsequent research on approximation capabilities of NNs \citep[e.g.,][]{hornik1991approximation,leshno1993multilayer,pinkus1999approximation}. These studies concluded that neural networks with sufficiently many hidden neurons possess the capacity to approximate any function within a broad class of functions. More literature has emerged in recent years concerning the approximation capabilities of neural networks, particularly deep networks that employ ReLU activation functions \citep[e.g.,][]{lu2017expressive,lin2018resnet}. 

\citet{white1990connectionist} examined NNs from a nonparametric perspective, showing that the approximation mentioned in \citet{hornik1989multilayer} are learnable by showcasing the consistency property of NNs, provided the complexity of the network increases with the sample size at specific growth rates. Work on consistency of neural networks continued \citep[]{mielniczuk1993consistency,lugosi1995nonparametric}. Further, \citet{barron1994approximation,mccaffrey1994convergence} developed error bounds for neural networks that demonstrate the rate of convergence. These studies laid the foundation for subsequent work on convergence rates, such as \citet{chen1999improved,kohler2005adaptive}, and \citet{shen2023asymptotic} established the asymptotic normality of neural network estimators along with their convergence rates. In recent years, discussions on learnability of NNs has broadened to encompass more than just a statistical perspective, increasingly incorporating an optimization viewpoint. This shift in focus has also seen a transition from exploring shallow neural networks to deep neural networks (DNNs)\citep[e.g.,][]{brutzkus2017globally,du2018gradient}.
The objective of these investigations is to elucidate the reasons why the networks trained using algorithms, such as stochastic gradient descent (SGD), achieve successful convergence. Furthermore, the prevalent applications of DNNs have prompted research into the convergence under the over-parameterized setting \citep[e.g.,][]{li2018learning,allen2019convergence}.

Another vital facet of theoretical work on neural networks has been centered around understanding their ability to extrapolate to new data. This ability, widely known as generalization, is evaluated by the generalization error, which quantifies a network's performance on unseen data. To establish generalization error bounds for NNs, abundant theories have focused on sample complexity, utilizing measures such as VC-dimension and Rademacher complexity \citep[e.g.,][]{baum1988size,blumer1989learnability,HAUSSLER199278,bartlett1996valid}. These proposed theoretical generalization bounds quantify the amount of data required to ensure accurate performance of an NN. \citet{bousquet2002stability} introduced the notions of stability for learning algorithms, providing an alternative pathway for deriving generalization bounds, which was followed by numerous subsequent studies such as \citet{hardt2016train}. The work of \citet{zhang2017understanding}, which implemented a randomization test challenging various prior studies, stimulated continued research on improving generalization bounds \citep[e.g.,][]{neyshabur2017exploring,golowich2018size}. Simultaneously, research developed under the over-parameterized setting has also gained traction in recent years \citep[e.g.,][]{arora2019fine,allen2019convergence}.

In this paper, we incorporate the framework studied in traditional statistic theory \citep{white1989learning} and offer a novel perspective on the question of why neural networks perform well, which goes beyond the characteristics of an individual NN. Commonly, the process of building a neural network necessarily involves the determination about two types of parameters. The first, known as model parameters, are gradually learned through the training process, while the second, referred to as hyperparameters, must be specified prior to training. Hyperparameters play a pivotal role in shaping the structure of neural networks and monitoring the learning process, contributing to substantial variability that allows neural networks to adapt to a diverse range of problems. Consequently, optimization of hyperparameters is a critical task in real-world applications of neural networks.

A variety of algorithms \citep{yang2020hyperparameter} can achieve hyperparameter optimization. A striking majority of these algorithms are reliant on sample splitting. Typically, after hyperparameters are set, it is necessary to continually evaluate the model’s performance during training. This process can then direct subsequent hyperparameter tuning to improve the model.
To do this evaluation, a validation set comprising samples that have not been used for training is often required, thus necessitating a split from the original dataset prior to training. In this work, our theoretical and empirical explorations underscore the importance of sample splitting, since this procedure enables us to attain a nearly optimal neural network from those with all alternative hyperparameters. Specifically, hyperparameters tuned by optimizing the performance on validation set are asymptotically optimal in the sense of minimizing prediction risk. Therefore, the practical success of neural networks can be attributed to sample splitting, a common practice in machine learning.

 
In fact, the efficacy of sample splitting can be recognized in several traditional statistical theories. For example, \citet{li1987asymptotic} and \citet{shao1997asymptotic} explored multiple procedures in the context of model selection. They derived the asymptotic optimality of these procedures, such as cross-validation, AIC, $C_p$, and other criteria, within the realm of linear models. The lack of explicit reference to sample splitting in these studies probably stems from the fact that the majority of these procedures do not necessarily require it. It is notable, however, that cross-validation, one of the procedures studied, is based on sample splitting, since it partitions the data into separate subsets for model training and evaluation purposes.
As such, sample splitting has the potential to account for the effectiveness of models from a model selection standpoint. This perspective coincides with the central thesis of this paper, which investigates the success of neural network models through the lens of sample splitting. It is worth mentioning that our theory is established based on a basic splitting procedure that segregates samples into one training set and one validation set, yet it remains flexible enough to accommodate more sophisticated methods such as cross-validation.

The remainder of our article is structured as follows. Section \ref{sec:2} introduces the mathematical modeling of NNs and their hyperparameters, along with the problem statement and the objective. In Section \ref{sec:3}, we present the main theoretical results, their conditions, and an interpretation from a perspective of error decomposition. Section \ref{sec:4} presents experiments designed to validate our findings. Conclusion and future work are drawn in Section \ref{sec:5}. Detailed proofs of our main results and additional discussions can be found in the Appendices.

\section{Preliminaries} \label{sec:2}
Suppose that we have $(n_1+n_2)$ independent and identically distributed (i.i.d.) observations $S=\{(X_i,Y_i)\}_{i=1}^{n_1+n_2}$, where $X_i$ is a vector of predictors and $Y_i$ is a scalar dependent variable. These samples can be partitioned into training set  $S_{train} = \{(X_i^{t},Y_i^{t})\}_{i=1}^{n_1}$ and validation set $S_{val}= \{(X_i^{v},Y_i^{v})\}_{i=1}^{n_2}$. The available set of hyperparameters of NNs, including continuous, discrete and categorical variables, is denoted by $\Lambda$. In this paper, we assume that $\Lambda$ is a finite set, i.e. $|\Lambda|<\infty$. The choices of categorical hyperparameters, e.g., the activation function and optimizer, are intrinsically finite. The domains of continuous and discrete hyperparameters are generally bounded in practical applications \citep{bergstra2011algorithms},  which results in finite choices of discrete ones. 
Regarding the rationality of assuming continuous hyperparameters, for instance, the learning rate, as belonging to a finite set, several frequently used hyperparameter optimization techniques consider a discretized continuous hyperparameter space and perform optimization on it, including grid search and random search \citep{bergstra2012random}. Consequently, it is reasonable to develop our theory under the assumption that $|\Lambda|<\infty$.

Given any $\lambda \in \Lambda$, a hypothesis space of neural networks is defined, and we denote a neural network in this space by $f_\lambda(\cdot,w): \mathbb{R}^{p} \rightarrow \mathbb{R}$, which represents a neural network with the hyperparameter $\lambda$ and the model parameter $w\in \mathbb{R}^{d(\lambda)}$. Here, the model parameter $w$ can interpreted as originating from weights connecting layers, vectors of biases for nodes, and other similar elements, which are then reshaped into column vectors and stacked on top of one another. The term $d(\lambda)$ represents the dimension of $w$, that is, the number of trainable variables in the model, and is certainly finite for any given $\lambda$. Therefore, the maximum dimension of model parameters under $\lambda\in\Lambda$, which is equal to $ \sup_{\lambda\in\Lambda} d(\lambda)$, is finite in this paper as $|\lambda|<\infty$, which is also independent of sample size. 

For any hyperparameter $\lambda$, we can train an NN model by $S_{train}$, achieving the parameter $\hat{w}_{n_1}(\lambda)$ and the trained neural network $f_{\lambda}(\cdot,\hat{w}_{n_1}(\lambda)):\mathbb{R}^{p} \rightarrow \mathbb{R}$. We define the validation loss as the mean squared loss on validation set.

\begin{align}
    L_{n_2}(\lambda) =\frac{1}{n_2} \sum_{i=1}^{n_2} \{f_{\lambda}(X_i^{v},\hat{w}_{n_1}(\lambda))-Y_i^{v}\}^2.
\end{align}
To optimize the hyperparameter, the ideal method is to select $\lambda$ that achieves the best performance on a validation set. Therefore, we define the selected optimal hyperparameter as
\begin{align}
\label{eq:hat_lambda}
    \hat{\lambda} = \underset{\lambda\in\Lambda}{\operatorname{argmin}}\ L_{n_2}(\lambda).
\end{align}
Note that the existence of $\hat{\lambda}$ can be guaranteed by $|\Lambda|<\infty$. 

In real-world applications, the training set size $n_1$ and the validation set size $n_2$ are often set according to a proportion, such as a $90\%$-$10\%$ split of the total sample size. This means that $n_1$ and $n_2$ are generally of the same order with respect to the total sample size. Thus, without loss of generality and for notational simplicity in the ensuing text, we assume $n_1=n_2=n$.
By adopting $\hat{\lambda}$ as the hyperparameter and completing the training process, we obtain a network $f_{\hat{\lambda}}(\cdot,w_{n}(\hat{\lambda}))$. Our objective now is to uncover properties of $\hat{\lambda}$ or $f_{\hat{\lambda}}(\cdot,w_{n}(\hat{\lambda}))$. Suppose we have a set of new i.i.d. observations, denoted as $S_{0} = \{(X_j^{0},Y_j^{0})\}_{j=1}^{n_0}$, but we do not have the knowledge of dependent variables. As We will demonstrate in the next section, $\hat{\lambda}$ is, in the asymptotic sense, optimal for making prediction for this new observation set $S_0$. 

In the subsequent sections, we use $\|\cdot\|$ to denote the $\ell_2$-norm of a vector and $\top$ to represent matrix transpose. Absolute constants, which may vary from line to line are denoted by $c$ and $C$. For notational simplicity, we abbreviate the validation sample $(X_i^{v},Y_i^{v})$ as $(X_i,Y_i)$ for $i=1,...,n$, when it does not cause ambiguity. Since the samples are i.i.d., we sometimes use $(X_0,Y_0)$ to represent any sample from the set $S_0$ also for notational convenience.





\section{Theoretical Results} \label{sec:3}
In this section, we present the asymptotic properties of the selected hyperparameter $\hat{\lambda}$. We define the following notations 
\begin{align}
    \label{eq:3} &L(\lambda;X,Y) =  \{f_{\lambda}(X,\hat{w}_n(\lambda))-E(Y|X)\}^2,\\
    \label{eq:4} &L_0(\lambda) = \frac{1}{|S_0|}\sum_{(X,Y)\in S_0}L(\lambda;X,Y)= \frac{1}{n_0}\sum_{j=1}^{n_0} L(\lambda;X_j^0,Y_j^0),\\
    \label{eq:5} \text{and}\qquad &R_0(\lambda) = E(L_0(\lambda)), 
\end{align}
where $|S_0|$ denotes the number of elements in set $S_0$. For hyperparameter $\lambda\in \Lambda$, $L(\lambda;X,Y)$ measures the squared loss between the conditional expectation $E(Y|X)$ and the network output $f_{\lambda}(\cdot,\hat{w}_n(\lambda))$ with input $X$; $L_0(\lambda)$ is the mean prediction loss on set $S_0$; and $R_0(\lambda)$ is the prediction risk, the expectation of $L_0(\lambda)$. Since observations in $S_0$ are i.i.d., we can reformulate $R_0(\lambda)=E\{f_{\lambda}(X,\hat{w}_n(\lambda))-E(Y|X)\}^2$, where $(X,Y)\in S_{val}\bigcup S_0 $ can represent any sample independent of $S_{train}$. 

\subsection{Optimality of the Selected Hyperparameter}
We present the assumptions that are required for our theoretical results. Unless otherwise specified, the following limiting processes are studied with respect to $n\rightarrow\infty$.

\begin{assumption}
\label{assumption:1}
For any $\lambda\in\Lambda$, there exist a parameter $w^*(\lambda)$, such that $\|\hat{w}_n(\lambda)-w^*(\lambda)\| = O_p(n^{-1/2})$.
\end{assumption}
Assumption \ref{assumption:1} imports a high-level condition ensuring that for any given hyperparameter $\lambda$, the trained parameters of the neural networks $\wnhat$ will have a limiting value $\wstar$. This assumption holds under some reasonable conditions; see conditions in Theorem 1 of \citet{white1989learning}. 

With the importation of $\wstar$, we can further introduce these notations 
\begin{align}
    &L^*(\lambda;X,Y) =  \{f_{\lambda}(X,w^*(\lambda))-E(Y|X)\}^2,\notag\\
    &L_0^*(\lambda) = \frac{1}{|S_0|}\sum_{(X,Y)\in S_0}L^*(\lambda;X,Y)= \frac{1}{n_0}\sum_{j=1}^{n_0} L^*(\lambda;X_j^0,Y_j^0),\notag\\
    \text{and}\qquad &R^*_0(\lambda) = E(L^*_0(\lambda)),\notag 
\end{align}
which are similar to the definitions of $L(\lambda;X,Y)$, $L_0(\lambda)$, and $R_0(\lambda)$ in (\ref{eq:3}), (\ref{eq:4}), and (\ref{eq:5}) except for substituting $\wstar$ for $\wnhat$.

\begin{assumption}
\label{assumption:2}
For any $(X,Y)\in S_{train}\bigcup S_{val}\bigcup S_{0}$ and for any $\lambda\in\Lambda$, (i) $E|f_{\lambda}(X,w^*(\lambda))-Y|$ and $E|f_{\lambda}(X,w^*(\lambda))-E(Y|X)|$ are both $O(1)$; (ii) $f_{\lambda}(X,w)$ is differentiable with respect to $w$, and there exists a constant $\rho$, such that 
\begin{align}
\label{eq:assumption2}
    E\sup_{w^{0}\in \mathcal{O}(w^*(\lambda),\rho)}\bigg\|\frac{\partial f_{\lambda}(X,w)}{\partial w}\Big|_{w=w^0}\bigg\|=O(1),
\end{align}
where $\mathcal{O}(w^*(\lambda),\rho)$ denotes an open ball with center $w^*(\lambda)$ and radius $\rho$.
\end{assumption}
Assumption \ref{assumption:2} makes some restrictions on boundedness and differentiability. Assumption \ref{assumption:2}(i) holds when $f_{\lambda}(X,w^*(\lambda))$, $Y$ and $E(Y|X)$ are bounded. Examples to illustrate the rationality of Assumption \ref{assumption:2}(ii) under some primitive conditions are given in Appendix \ref{appendix:1}. 

Let $\xi_n = \inf_{\lambda\in\Lambda} R^*_0(\lambda)$ denote the minimum risk among all $\lambda \in \Lambda$ with neural networks using limiting value of parameters $\wstar$, and we make these three assumptions:

\begin{assumption}
\label{assumption:3} $\Delta_n$ is integrable, where
$$
\Delta_n=\xi_n^{-1}\sup_{\lambda\in\Lambda}\big|\{f_{\lambda}(Y_0,\hat{w}_n(\lambda))-Y_0\}^2-\{f_{\lambda}(Y_0,w^*_n(\lambda))-Y_0\}^2\big|.
$$
\end{assumption}

\begin{assumption}
\label{assumption:4} For any $\lambda\in\Lambda$, $\xi_n^{-1}n^{-1/2} = o(1)$.
\end{assumption}

\begin{assumption}
\label{assumption:5} For any sample $(X,Y)\in S_{val}$ and any $\lambda\in\Lambda$ that (i) $\operatorname{var}[f_\lambda(X,w^*(\lambda))\{Y-E(Y|X)\}] =O(1)$; (ii) $\operatorname{var}[\{f_{\lambda}(X,w^*(\lambda))-E(Y|X)\}^2] = O(1)$.
\end{assumption}
As shown in (\ref{eq:14}) in the proof of Theorem \ref{theorem:1} that $\Delta_n$ is $o_p(1)$, Assumption \ref{assumption:3} is only imposed to ensure the expectation of $\Delta_n$ is o(1). Assumption \ref{assumption:4} is similar to Condition (11) in \citet{ando2014model}, Condition (C.3) of \citet{zhang2016optimal}, and Assumption 2.3 in \citet{Liu2013heteroscedasticity}, which makes a restriction on the minimum risk $\xi_n$, requiring the speed of its converging to zero to be not too fast---at least slower than $n^{-1/2}$. 
As a concrete example, Assumption \ref{assumption:4} is satisfied if $c\leq\xi_n$ for any $n$ for some constant $c$. Such cases arise when the hypothesis space, defined by any hyperparameter $\lambda\in\Lambda$, is insufficiently capable of approximating the objective function. Conversely, Assumption \ref{assumption:4} is violated if $\xi_n\equiv0$ for any $n\geq0 $. See interpretation from the perspective of error decomposition in Section \ref{sec:3.2}.Assumption \ref{assumption:5} has regularity conditions ensuring certain terms are not heavy-tailed so that their variance can be bounded. Now we are ready to present our first theorem. 

\begin{theorem}
\label{theorem:1}
Under Assumptions \ref{assumption:1}-\ref{assumption:5}, we have
\begin{align}
\label{eq:theorem1}
    \frac{R_0(\hat{\lambda})}{\inf_{\lambda\in\Lambda} R_0(\lambda)}  = 1+ O_p(\xi_n^{-1}n^{-1/2}) = 1 + o_p(1).
\end{align}
\end{theorem}
Theorem \ref{theorem:1} indicates that the selected hyperparameter $\hat{\lambda}$ is optimal in the sense that the prediction risk of the neural network using hyperparameter $\hat{\lambda}$ is asymptotically identical to the risk of the infeasible best neural network. Note that in (\ref{eq:theorem1}), since $\hat{\lambda}$ is directly plugged in the expressions $R_0(\lambda)$, the randomness of $\hat{\lambda}$ is not considered. This plug-in is done after taking the expectation in $R_0(\hat{\lambda})$, which makes (\ref{eq:theorem1}) independent of test size $n_0$, and hence Theorem \ref{theorem:1} holds for arbitrary $n_0$. In what follows, our second result shows that we can build other forms of optimalities based on $L_0(\lambda)$ with some additional assumptions. 

\begin{assumption}
\label{assumption:6} For any $\lambda\in\Lambda$, $\xi_n^{-1}n_0^{-1/2} = o(1)$.
\end{assumption}
Assumption \ref{assumption:6} is similar to Assumption \ref{assumption:4} except that the large sample condition is also with respect to $n_0$. If we assume that $n_0$ and $n$ are of same order, Assumption \ref{assumption:4} and Assumption \ref{assumption:6} are actually equivalent. As $L_0(\lambda)$ denotes the mean prediction loss on the test set $S_0$, this measure can be affected by the random fluctuations of samples. Therefore, it is required that the size of the test set, $n_0$, to be divergent, as assumed in \ref{assumption:6}.

\begin{assumption} 
    \label{assumption:7}
    (i) There exists a constant $C$ such that $L_0(\hat{\lambda})-\inf_{\lambda \in \Lambda} L_0(\lambda)<C$; (ii) $\{L_0(\hat{\lambda})-\inf _{\lambda \in \Lambda} L_0(\lambda)\}/\inf _{\lambda \in \Lambda} L_0(\lambda)$ is uniformly integrable.
\end{assumption}
Assumption \ref{assumption:7}(i) holds when $f_{\lambda}(X,w^*(\lambda))$ and $E(Y|X)$ are bounded. Assumption \ref{assumption:7}(ii) is imposed to obtain the $L_1$ convergence of the term on the left hand side of (\ref{eq:7}) in Theorem \ref{theorem:2}. 
\begin{theorem}
\label{theorem:2}
For a diverging $n_0$, under Assumptions \ref{assumption:1}-\ref{assumption:6}, we have
\begin{align}
    \label{eq:7}
    \frac{L_0(\hat{\lambda})}{\inf_{\lambda\in\Lambda} L_0(\lambda)}= 1+O_p(\xi_n^{-1}n^{-1/2}+\xi_n^{-1}n_0^{-1/2}) = 1 + o_p(1).
\end{align}
Furthermore, if Assumption \ref{assumption:7} holds, then we have
\begin{align} 
    \label{eq:ad1}
    E\Big\{\frac{L_0(\hat{\lambda})}{\inf _{\lambda \in \Lambda} L_0(\lambda)}\Big\}=1+o(1) 
\end{align}
and
\begin{align} \label{eq:ad2}
    \frac{E\{L_0(\hat{\lambda})\}}{E\{\inf _{\lambda \in \Lambda} L_0(\lambda)\}}=1+o(1). 
\end{align}
\end{theorem}
Theorem \ref{theorem:2} contains three types of optimalities. Equation (\ref{eq:7}) concerns the optimality in terms of asymptotically minimizing $L_0$, the mean prediction loss on test set $S_0$, which is established based on the result in Theorem \ref{theorem:1}. Equation (\ref{eq:ad1}) and (\ref{eq:ad2}) concern optimalities involving different forms of expectation.
It is important to note that $E\{L_0(\hat{\lambda})\}$ is distinct from $R_0(\hat{\lambda})$ presented in Theorem \ref{theorem:1}, since the randomness of $\hat{\lambda}$ is taken into consideration in (\ref{eq:ad1}) and (\ref{eq:ad2}) but not in $R_0(\hat{\lambda})$ . Consequently, the optimalities outlined in (\ref{eq:ad1}) and (\ref{eq:ad2}) is more reflective of real-world scenarios compared to Theorem \ref{theorem:1}, thereby offering theoretical evidence in practical application. To verify the justifications in Theorem \ref{theorem:2}, we will carry out a series of experimental studies in Section \ref{sec:4}. 

\subsection{Error Decomposition Interpretation} \label{sec:3.2}
For a more precise comprehension of theory, we revisit the objective in this paper from the perspective of error decomposition in machine learning theory. The optimalities of $\hat{\lambda}$ in Theorem \ref{theorem:1} and Theorem \ref{theorem:2} are obtained in the sense of minimizing $L_0(\lambda)$ or $R_0(\lambda)$, which both involve the difference between the target $E(Y|X)$ and the trained model $f_{\lambda}(X,\hat{w}_n(\lambda))$. This discrepancy can be written as 
\begin{align}
     & E(Y|X)-f_{\lambda}(X,\hat{w}_n(\lambda)) \notag\\ 
    =& \{E(Y|X)-f_{\lambda}(X,w^*(\lambda))\} + \{f_{\lambda}(X,w^*(\lambda))-f_{\lambda}(X,\hat{w}_n(\lambda))\}. \notag
\end{align}
Now we preceive $f_{\lambda}(X,w^*(\lambda))$ as a quasi-optimal estimate of $E(Y|X)$ within models defined by hyperparamter $\lambda$, and then we can interpret the two terms respectively as 1) the model error $\{E(Y|X)-f_{\lambda}(X,w^*(\lambda))\}$ (also referred to as approximation error)---the error caused by the inadequate hypothesis space of the model determined by the choice of $\lambda$; and 2) the estimation error $\{f_{\lambda}(X,w^*(\lambda))-f_{\lambda}(X,\hat{w}_n(\lambda))\}$---the error resulting from the gap, given $\lambda$, between the quasi-optimal parameter $w^*(\lambda)$ and the parameter $\hat{w}_n(\lambda)$ derived from only $n$ training samples. If we revisit the definition of $L_0^*(\lambda)$, the model error $\{E(Y|X)-f_{\lambda}(X,w^*(\lambda))\}$ corresponds to the term within $L_0^*(\lambda)$'s square. Therefore, any assumptions imposed on $L_0^*(\lambda)$ or $R_0^*(\lambda)$ (such as Assumption \ref{assumption:4}, \ref{assumption:6}, and \ref{assumption:7}) inherently carry assumptions about the model error. Assumption \ref{assumption:1} imports a condition ensuring the convergence of $\hat{w}_n(\lambda)$ to $w^*(\lambda)$ at a certain rate. We leverage this condition to manage the estimation error $\{f_{\lambda}(X,w^*(\lambda))-f_{\lambda}(X,\hat{w}_n(\lambda))\}$. This indicates that our work contemplates both model error and estimation error from the perspective of error decomposition. The optimalities established in our theorems can be viewed as accounting for the effectiveness of sample splitting with the modeling of both errors. 

Note that the optimization error---the error incurred during training process---is not our main focus, which is consistent with previous work that study neural networks from a statistical perspective \citep[e.g.,][]{white1989learning,barron1994approximation,benedikt2019deep}. By utilizing Assumption \ref{assumption:1}, we have addressed the estimation error, as mentioned in the previous paragraph, while concurrently avoided the explicit modeling of the optimization error. By invoking the convergence property of $\hat{w}_n(\lambda)$ under Assumption \ref{assumption:1}, we directly perceive $\hat{w}_n(\lambda)$ as a set of satisfactory network parameters obtained under $n$ training samples, thereby dispensing with the need to model the specific training process of obtaining them. The specific modeling of these elements---the training process and optimization error---might well serve as a avenue for future research.

\section{Experimental Verification} \label{sec:4}
In this section, we consider several common neural network structures and use them to model a range of diverse problems, including linear/nonlinear regression, binary-classification, and time series forecasting, to verify our theoretical results in Theorem \ref{theorem:2}.

\begin{table}[htbp]
    \centering
    \begin{tabular}{@{}llcccc@{}}
        \toprule
        & Scenario & LR & HS & Depth & BS \\
        \midrule
        \multirow{3}{*}{MLP} & Linear & 0.1, 0.01, 0.001 & 5, 10, 20 & 1, 2 & 8, 16, 32 \\
        & Nonlinear & 0.01, 0.001 & 50, 100 & 1, 2, 3 & 16, 32, 64 \\
        & Classification & 0.001, 0.0001 & 50, 100 & 1, 3, 5 & 16, 32, 128 \\
        \bottomrule
    \end{tabular}
    \caption{Design of $\Lambda$ for experiments with MLPs}\label{tab:1}
\end{table}

\subsection{Multilayer Perceptrons}
We conduct experiments in three scenarios based on a basic NN structure multilayer perceptrons (MLPs), which follows a fully connected structure. The first and second scenarios are respectively based on linear and nonlinear regression problems. The third scenario simulates a classification problem. The training and validation set sizes are set to $n_1=n_2=n$, with $n$ gradually increasing from $50$ to $50000$, while the test set size $n_0=n/5$. The set of hyperparameters, denoted by $\Lambda$, consists of four types varying hyperparameters: LR (learning rate), HS (number of hidden nodes in each layer), depth (number of hidden layers), and BS (batch size), with the alternative ranges of these hyperparameters in different scenarios shown in Table \ref{tab:1}. All other hyperparameters are fixed as follows. The maximum epoch was set to 50, the activation function was the Relu function, and the optimizer used was Adam. The Mean Squared Error (MSE) loss function is used for training regression models, while the commonly used Cross Entropy (CE) loss function is used for training classification models. Now we will present the specific designs and results for each scenario.

\subsubsection{Linear Regression} \label{sec:4.1.1}
For the first scenario in our simulations, we consider a linear relationship between $X$ and $Y$. 
\begin{align}
    Y_i = X_i^\top \beta + \epsilon_i= \sum_{l=1}^5 \beta_lX_{il}+\epsilon_i, \notag \quad i=1,...,n,
\end{align}
where $\epsilon_i\sim N(0,\sigma^2)$, $X_i=(X_{i1},...,X_{i5})^\top\sim N(\mathbf{0},\Omega)$, with the diagonal elements of $\Omega$ set to 1 and off-diagonal elements set to 0.5. We set the coefficient vector $\beta=(\beta_1,...,\beta_5)^{\top}$ as $\beta_l=1$ for $1\leq l\leq3$ and $\beta_l=-1$ for $4\leq l\leq5$. The variance of the noise, denoted by $\sigma^2$, is set to vary between $\sigma^2=1/3,1,3$, corresponding to $R^2=0.9,0.75,0.5$, respectively. The hyperparameters were set according to the `Linear' row in Table \ref{tab:1}.

For each value of $n=50,...,50000$, we repeat the same simulation for 50 times. In each simulation replication, we train a number of neural networks with every hyperparameter choice $\lambda \in \Lambda$. After the training process is completed, we select the best hyperparameters by $\hat{\lambda} = \operatorname{argmin}_{\lambda\in\Lambda}\ L_n(\lambda)$, record $L_0(\lambda)$ for every $\lambda \in \Lambda$, and then calculate $\inf_{\lambda\in\Lambda}L_0(\lambda)$. Note that $E(Y|X)=X^\top \beta$ when calculating $L_0(\lambda)$ here. To verify (\ref{eq:ad1}), we calculate the ratio $L_0(\hat{\lambda})/\inf_{\lambda\in\Lambda}L_0(\lambda)$ in all replications. Figure \ref{fig:1a} presents the average value of this ratio across all replications to represent $E\{L_0(\hat{\lambda})/\inf_{\lambda\in\Lambda}L_0(\lambda)\}$. To verify (\ref{eq:ad2}), we separately calculate the average values of $L_0(\hat{\lambda)}$ and $\inf_{\lambda\in\Lambda}L_0(\lambda)$ separately across all replications and then divide these two values to represent the ratio $E\{L_0(\hat{\lambda)}\}/E\{\inf_{\lambda\in\Lambda}L_0(\lambda)\}$. Figure \ref{fig:1b} presents this ratio under different settings. The results presented in Figure \ref{fig:1} demonstrate that both $E\{L_0(\hat{\lambda})/\inf_{\lambda\in\Lambda}L_0(\lambda)\}$ and $E\{L_0(\hat{\lambda)}\}/E\{\inf_{\lambda\in\Lambda}L_0(\lambda)\}$ converge to 1 regardless of the choice of $\sigma^2$, manifesting the validity of Theorem \ref{theorem:2} in this scenario.
\begin{figure}[htbp]
    \centering
    \subfigure[]{
    \label{fig:1a}
    \includegraphics[width=0.47\linewidth]{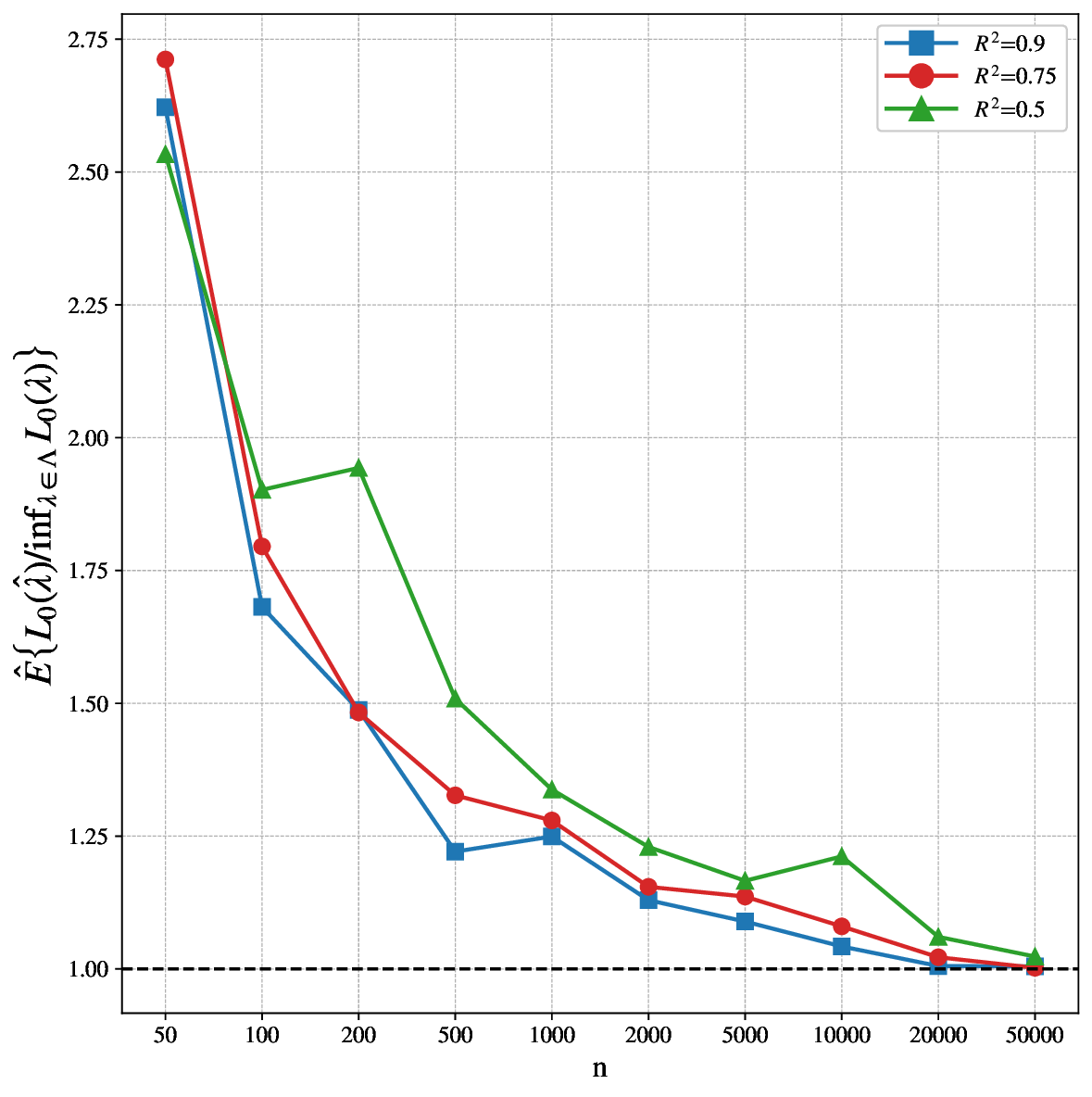}}
    \subfigure[]{
    \label{fig:1b}
    \includegraphics[width=0.47\linewidth]{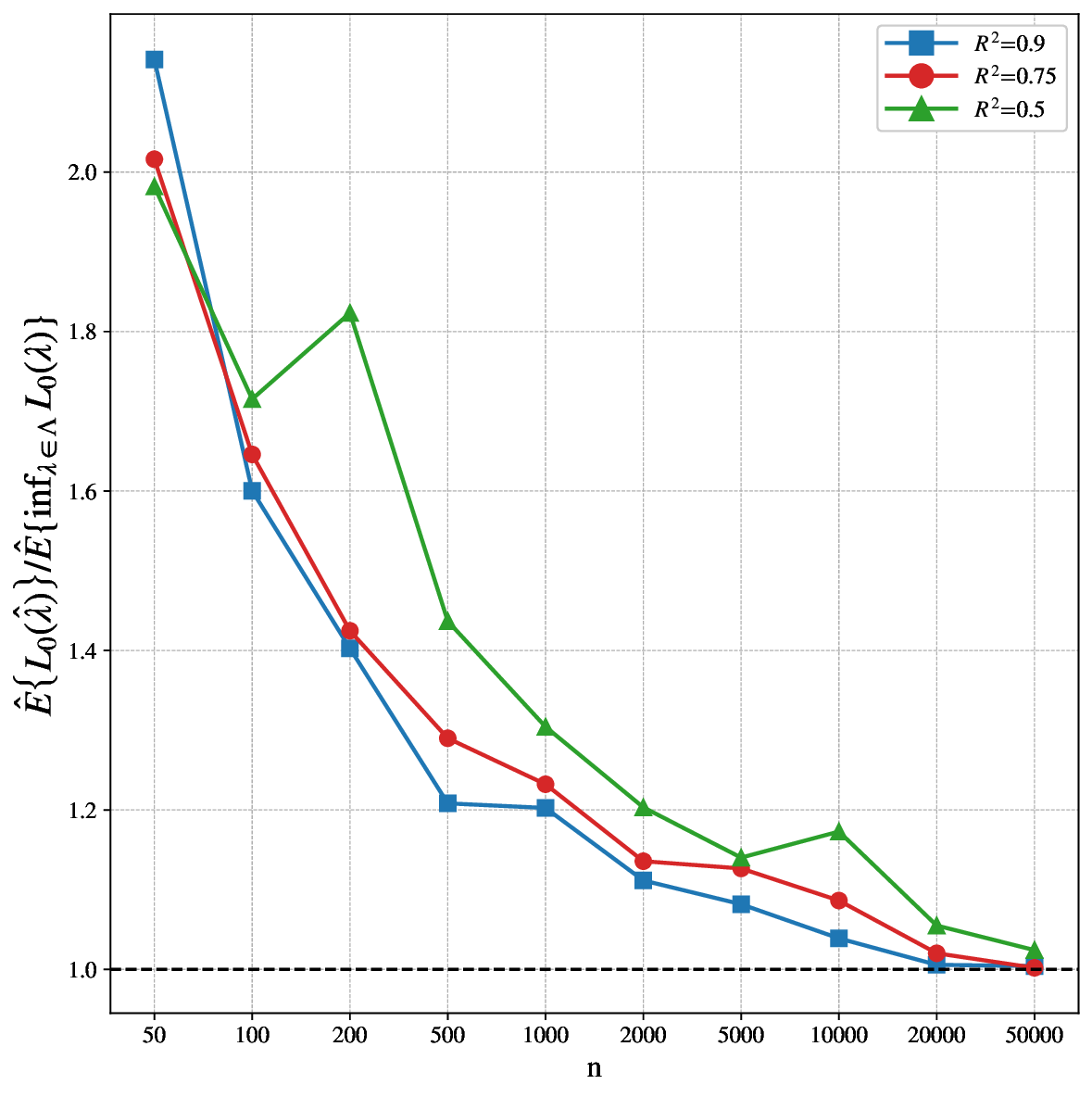}}
    \caption{Results of MLPs for linear regression. (a) The estimated value of $E\{L_0(\hat{\lambda})/\inf_{\lambda\in\Lambda}L_0(\lambda)\}$. (b) The estimated value of $E\{L_0(\hat{\lambda)}\}/E\{\inf_{\lambda\in\Lambda}L_0(\lambda)\}$. }
    \label{fig:1}
\end{figure}


\subsubsection{Nonlinear Regression} \label{sec:4.1.2}
For the second simulation design, we consider a non linear regression problem of the form:
\begin{align}
    Y_i = \frac{10\sin{(\|X_i\|_2})}{\|X_i\|_2} + \epsilon_i,\quad i=1,...,n, \notag
\end{align}
where $\epsilon_i\sim N(0,\sigma^2)$ and ${X}_i=(X_{i1},...,X_{i5})^{\top}$. For $l=1,...,5$, $X_{il}$ is independent of each other and follows a uniform distribution between $-10$ and $10$. An additional simulation was conducted to approximate the variance of $\sin{(|X_i|_2)}/|X_i|_2$.
Accordingly, the choice of $\sigma$ are set to ensure $R^2=0.9, 0.75, 0.5$, respectively. The varying hyperparameters are set according to the `Nonlinear' row in Table \ref{tab:1}, and the training and analysis process is conducted in the same manner as that in our linear regression scenario. Note that $E(Y|X)=10\sin{(\|X\|_2})/\|X\|_2$ when calculating $L_0(\lambda)$ here. 

Figure \ref{fig:2a} and \ref{fig:2b} show that $E\{L_0(\hat{\lambda})/\inf_{\lambda\in\Lambda}L_0(\lambda)\}$ and $E\{L_0(\hat{\lambda)}\}/E\{\inf_{\lambda\in\Lambda}L_0(\lambda)\}$ converge to 1 in nonlinear settings as well.
\begin{figure}[htbp]
    \centering
    \subfigure[]{
    \label{fig:2a}
    \includegraphics[width=0.47\linewidth]{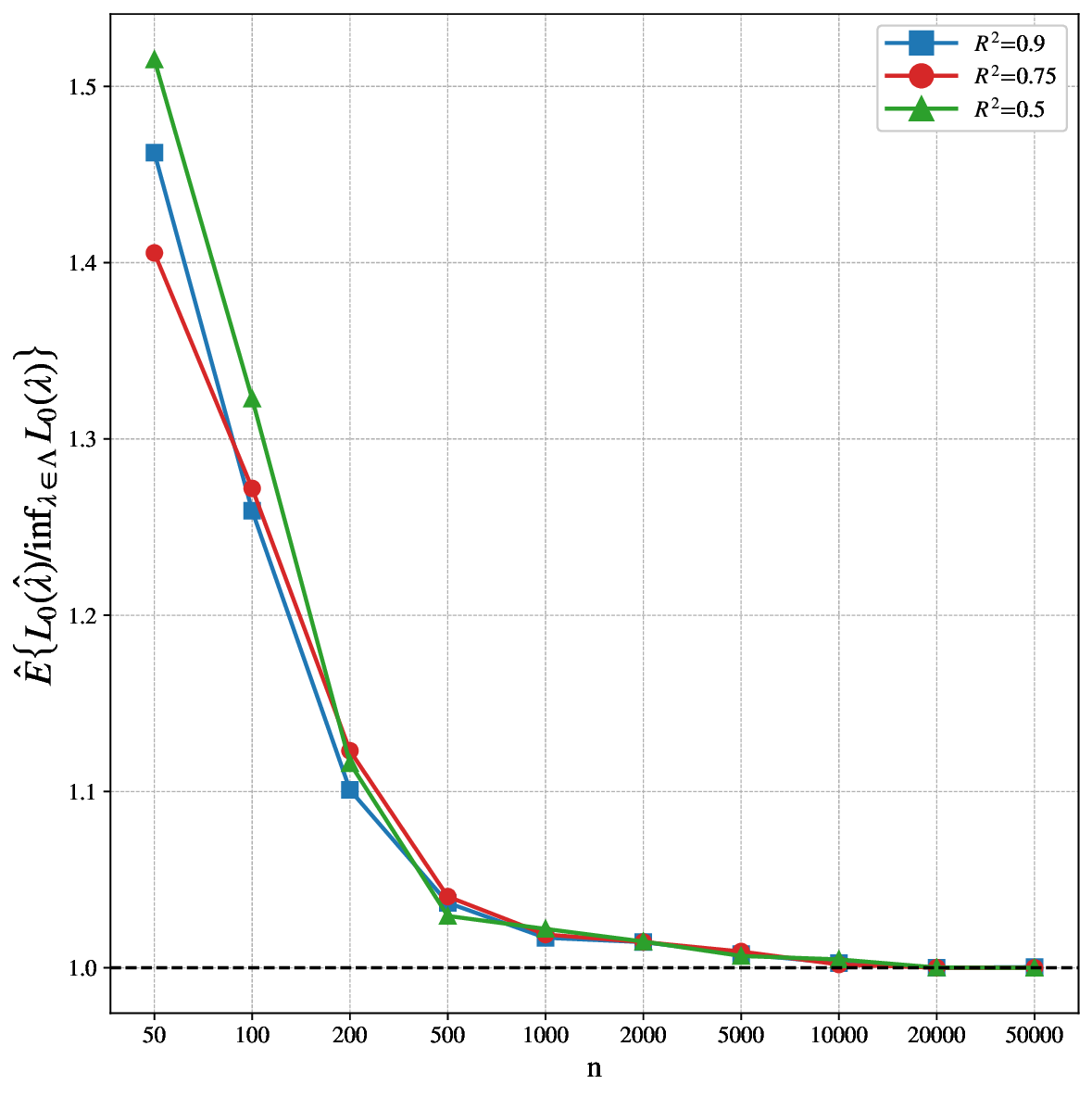}}
    \subfigure[]{
    \label{fig:2b}
    \includegraphics[width=0.47\linewidth]{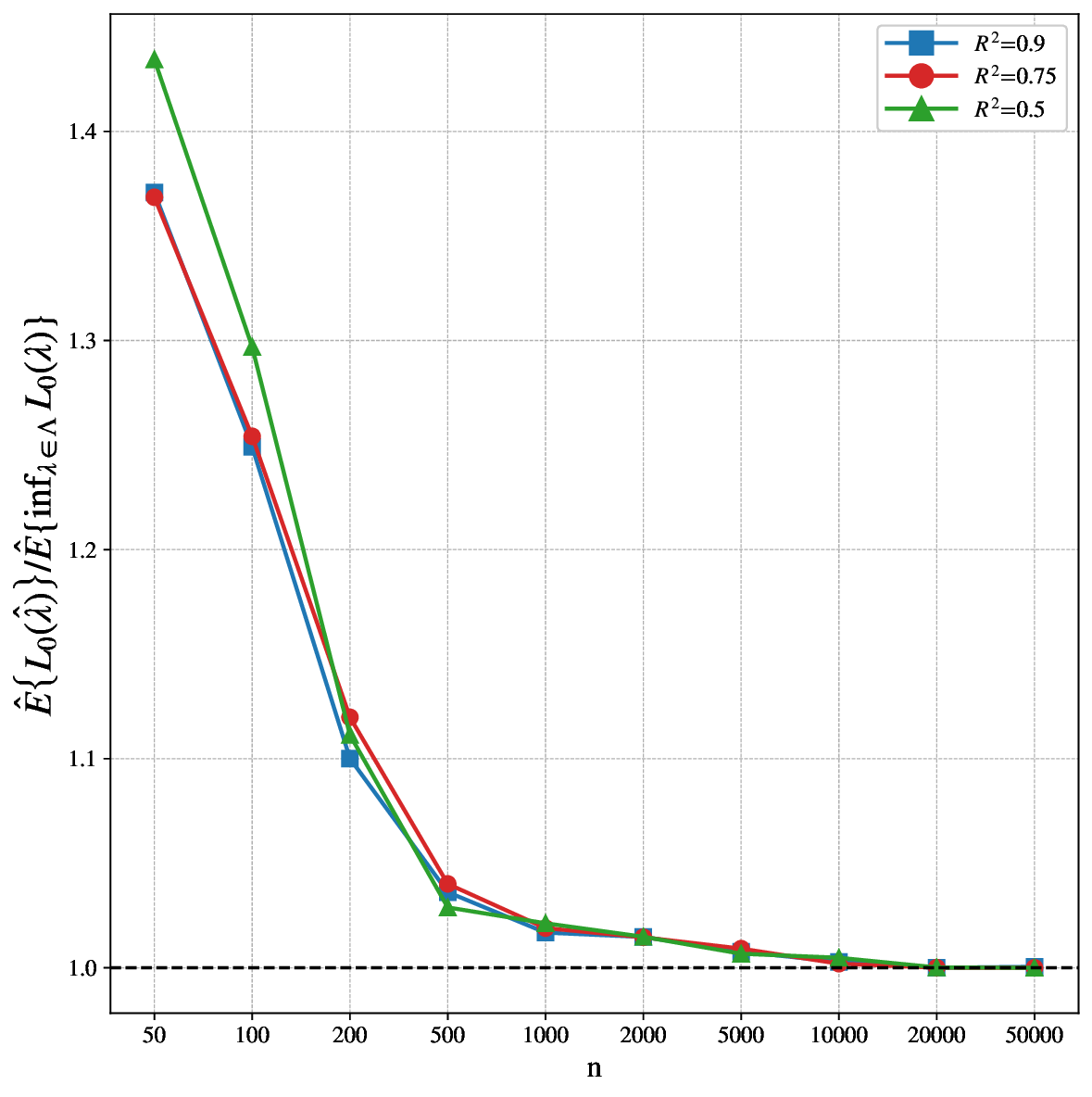}}
    \caption{Results of MLPs for nonlinear regression. (a) The estimated value of $E\{L_0(\hat{\lambda})/\inf_{\lambda\in\Lambda}L_0(\lambda)\}$. (b) The estimated value of $E\{L_0(\hat{\lambda)}\}/E\{\inf_{\lambda\in\Lambda}L_0(\lambda)\}$.}
    \label{fig:2}
\end{figure}

\subsubsection{Classification} \label{sec:4.1.3}
For the third simulation design, we address a binary classification problem. Specifically, we assume that $Y_i$ takes values of 0 or 1, and the conditional probability of $Y_i$ being 1 given $X_i$, denoted as $\operatorname{P}(Y_i=1|X_i)$, is given by $g(X_i)$:
\begin{align}
    Y_i = 0, 1, \quad \operatorname{P}(Y_i=1|X_i) = g(X_i)\notag,\quad i =1,...,n.
\end{align}
Here $g(\mathbf{x})=1/[1+\exp\{-(h(\mathbf{x})-\mu)^2\}]$, $\mathbf{x}=(x_1,...,x_{10})^\top\in R^{10}$, where $h(\mathbf{x})$ is a nonlinear function of $\mathbf{x}$ defined as 
$$
\begin{aligned}
h(\mathbf{x}) &=12 x_{1}(x_{2}-0.5)^{2}-16\{x_{3}(x_{5}-0.2)\}^{4} \\
&+2 \log \{3-4\left(x_{4}-0.3\right)^{2}+x_{5}+\exp (-x_{6} x_{7}+x_{5})\} \\
&+2 \tan [4\{(x_{1}(x_{8}-0.5)\}^{2}+0.1]
\end{aligned}
$$
with $x_{9}$ and $x_{10}$ being redundant. We generate the an sample input as $X_i = (X_{i1},...,X_{i10})^{\top}$, where $X_{il}$ are independent and uniformly distributed between 0 and 1 for $l=1,...,10$. We set $\mu$ to ensure that $E(h(X_i))=\mu$, and $\Lambda$'s setting is shown in `Classification' row in Table \ref{tab:1}. The activation function is Relu, and we use a sigmoid function before the output layer to limit the range of output to interval $(0,1)$.  

The loss function for training is cross-entropy (CE). However, for consistency with our theory, the determination of $\hat{\lambda}$ through the validation set and subsequent verification of the theorem are based on mean squared error (MSE) instead of CE. Specifically, $\hat{\lambda}$ is still chosen according to (\ref{eq:hat_lambda}), i.e., by minimizing the MSE between real-valued $f_{\lambda}(X,\hat{w}_n(\lambda))$ and one-hot value $Y$. In terms of computing $L_0(\lambda)$, since we have $E(Y|X)=P(Y=1|X)=g(X)$, we calculate the square error of real-valued $f_{\lambda}(X,\hat{w}_n(\lambda))$ and real-valued $g(X)$. 
Although this procedure leads to inconsistency between the loss function for training and the performance measurement, we state that the theoretical results are not affected as the properties that form during the training process are related to Assumption \ref{assumption:1} and will not affect subsequent reasoning.
Therefore, the optimality of $\hat{\lambda}$ still holds, and the simulation results also demonstrate this statement. We record the data in replicate simulations to approximate $E\{L_0(\hat{\lambda})/\inf_{\lambda\in\Lambda}L_0(\lambda)\}$ and $E\{L_0(\hat{\lambda)}\}/E\{\inf_{\lambda\in\Lambda}L_0(\lambda)\}$. Figure \ref{fig:3} shows these ratios also converge to 1 in this scenario, which is consistent with Theorem \ref{theorem:2}.

\begin{figure}[htbp]
    \centering
    \subfigure[]{
    \label{fig:3a}
    \includegraphics[width=0.47\linewidth]{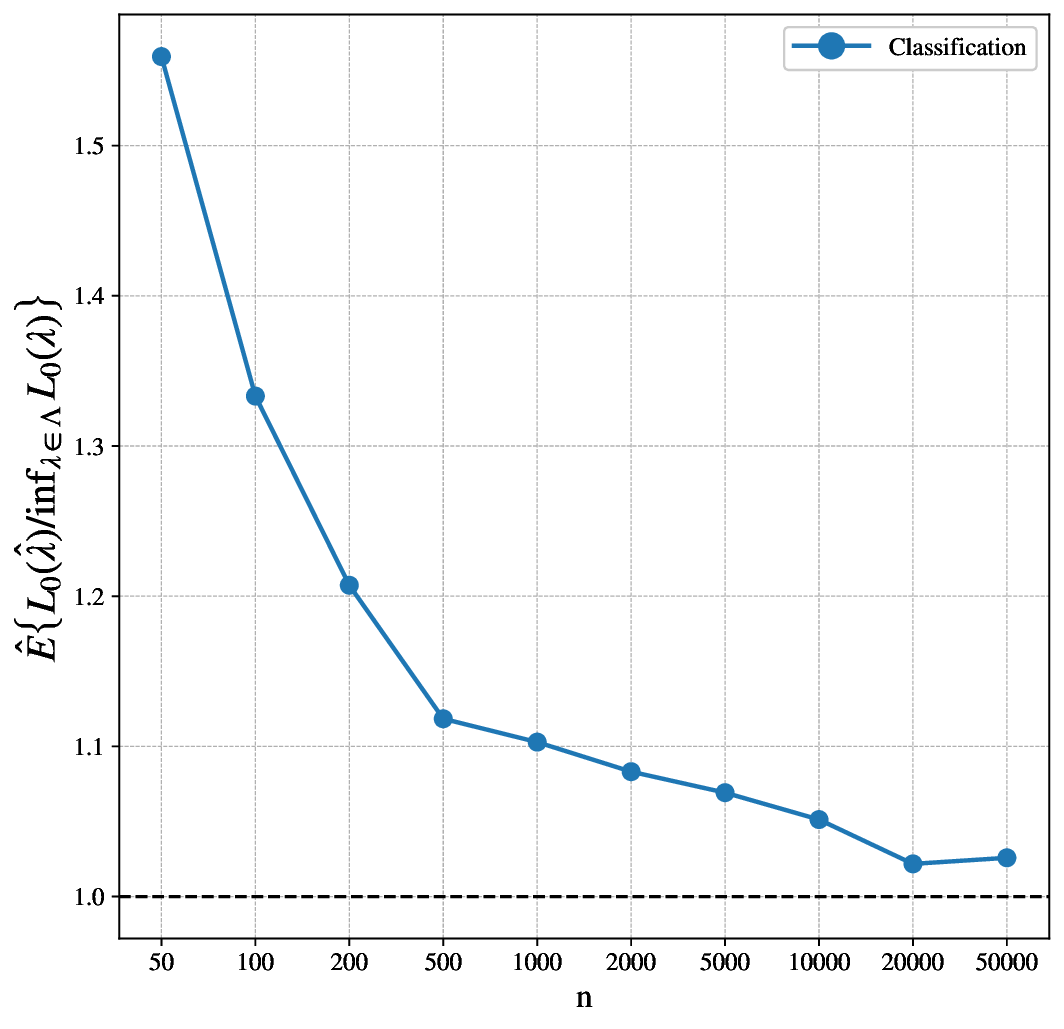}}
    \subfigure[]{
    \label{fig:3b}
    \includegraphics[width=0.47\linewidth]{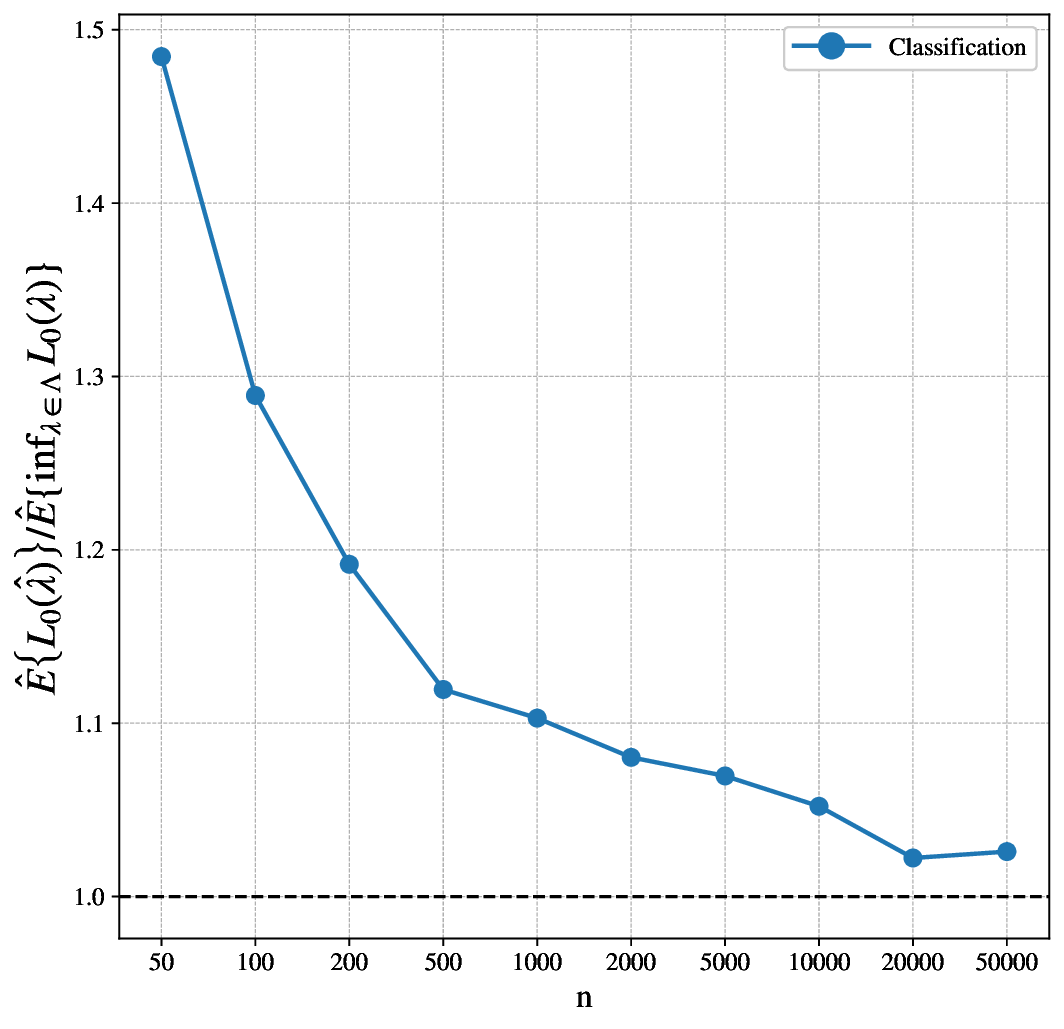}}
    \caption{Results of MLPs for binary classification. (a) The estimated value of $E\{L_0(\hat{\lambda})/\inf_{\lambda\in\Lambda}L_0(\lambda)\}$. (b) The estimated value of $E\{L_0(\hat{\lambda)}\}/E\{\inf_{\lambda\in\Lambda}L_0(\lambda)\}$.}
    \label{fig:3}
\end{figure}

Note that our experimental analysis is focused exclusively on binary classification problems since our theoretical framework can not model the expectation of labels in a multi-class classification problem. As such, the experiments we conduct in the following sections regarding classification will be restricted to binary classification problems.

\subsection{Convolutional Neural Networks} \label{sec:4.2}
In this section, we carry out experiments involving convolutional neural networks (CNNs), a specialized form of feedforward neural networks that exhibit distinct connection structures compared to MLPs. Since CNNs are often employed when processing image-based inputs, the data generating process in previous simulations is not applicable to this context. Thus, we decide to utilize widely-used public datasets for image recognition tasks. However, we are faced with a problem that a mathematical formulation representing the conditional expectation $E(Y|X)$ is unavailable in real-world datasets. This limitation prevents us from calculating the value of any $L_0(\lambda)$ in our theorem. To address this issue, we design a special method to generate new samples from the MNIST/Fashion-MNIST datasets as follows, ensuring that the theoretical value of $E(Y|X)$ remains attainable:

\begin{itemize}
    \item Define a binary-class classification problem for the MNIST/Fashion-MNIST dataset. 
    \item Change the labels of samples in the dataset to $1$ and $0$ to represent positive and negative samples of the binary-class classification problem, respectively.
    \item Train a well-performed classification model $F:\mathbb{R}^{28\times28}\rightarrow [0,1]$ with these samples.
    \item Relabel the samples by using the trained model $F$. Specifically, for any $X_i$ from a sample in the dataset, we generate its new label $Y_i$ by
    \begin{align}
     \operatorname{P}(Y_i=1|X_i) = F(X_i). \notag
    \end{align}
\end{itemize}
This method treats $F(X)$ as the true probability of input $X$ belonging to the positive class. For example, if an input $X$ satisfies $F(X)=0.75$, we generate a new sample $(X,Y)$ with $\operatorname{P}(Y=1|X) = 0.75$, and consequently, it is evident that $E(Y|X)=0.75$. 
Through this procedure, we can transform the common classification task on a real-world dataset into a binary-classification problem, and possess true knowledge of $E(Y|X)$ to verify our theorem.

\subsubsection{Binary-Classification on Fashion-MNIST} \label{sec:4.2.1}

We first describe the design of the binary-classification problem on Fashion-MNIST. Specifically, chothing that covering the upper human body is set to be Class 0, which involves label sets [0,2,3,4,6], and the rest of labels are Class 1, including clothing such as thousands, shoes, and bags. Adopting the data generating process and training procedure introduced in Section \ref{sec:4.2}, we create a new dataset derived from fashion-MNIST with labels converted to either 0 or 1 and train a well-performed model $F_1$ on this task. See Appendix \ref{appendix:3} for details of $F_1$.
 
We consider training a set of CNNs with a standard architecture consisting of two convolutional layers and two max pooling layers, followed by two fully connected layers that produce the final output. The choices of varying hyperparamters in $\Lambda$ are shown in `Fashion-MNIST' row in Table \ref{tab:2}, where Conv KS refers to kernel size of the squared convolutional layer, Chan Out refers to number of output channels by the convolution (number of convolution filters), Pool KS is kernel size of the squared pooling layer, and Pool Str is stride of pooling. Besides, learning rate, batch size, activation function and loss function are set to 0.0005, 16, Relu and CE respectively. The hidden size of the fully connected layer is set to 128. With $n$ taking values from $50$ to $10000$, we perform the same analysis as in the MLP classification problem described in Section \ref{sec:4.1.3}. Results given in Figure \ref{fig:fashion} show Theorem \ref{theorem:2} also works in this scenario.



\begin{table}[htbp]
    \centering
    \begin{tabular}{lccc}
        \toprule
        Scenario & (Conv KS, Chan Out) & Pool KS & Pool Str \\
        \midrule
        Fashin-MNIST & (3,4),(3,16),(3,64),(4,4),(4,16),(4,64),(5,4),(5,16) & 2,3 & 1,2 \\
        MNIST & (3,4),(3,16),(3,64),(4,16),(4,64) & 2,3 & 1,2 \\
        \bottomrule
    \end{tabular}
    \caption{Design of $\Lambda$ for experiments with CNNs}  \label{tab:2}
\end{table}

\begin{figure}[htbp]
    \centering
    \subfigure[]{
    \label{fig:fashiona}
    \includegraphics[width=0.47\linewidth]{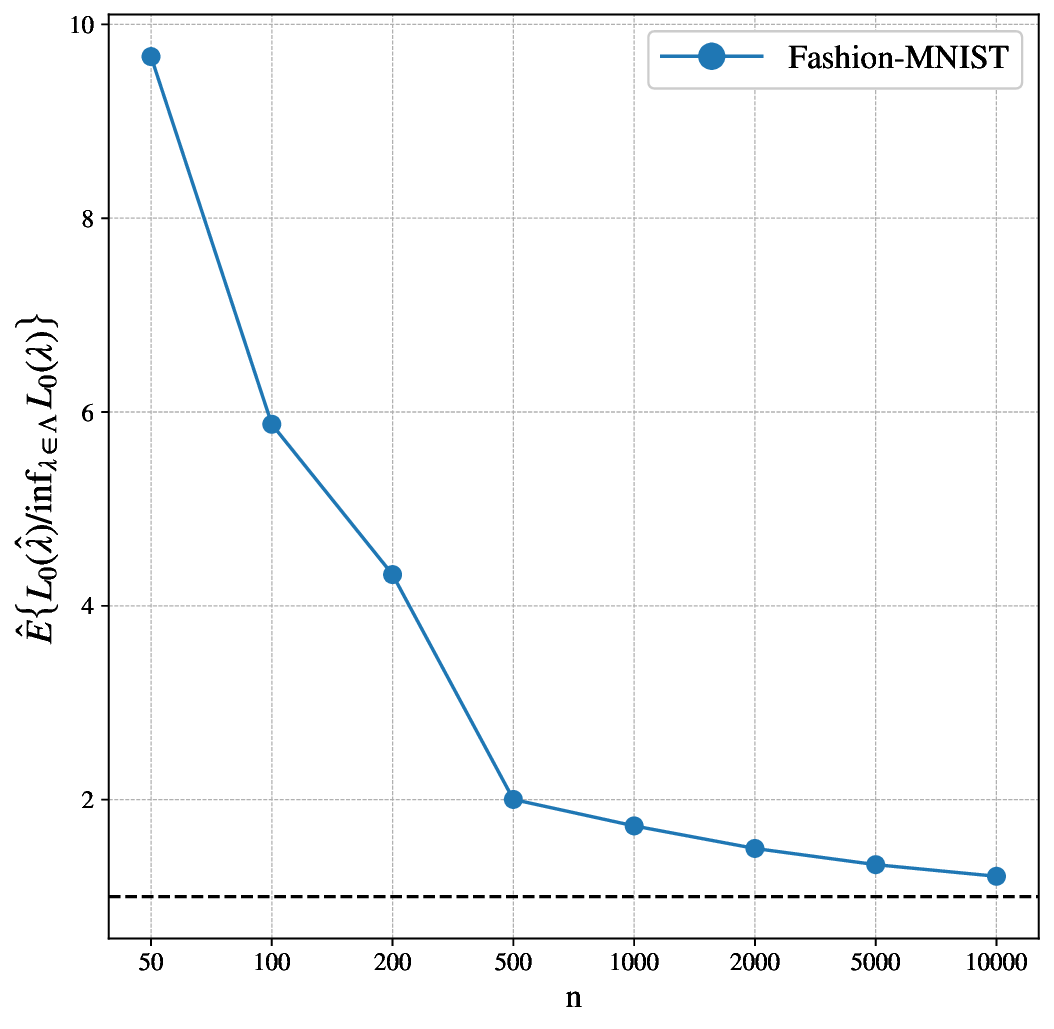}}
    \subfigure[]{
    \label{fig:fashionb}
    \includegraphics[width=0.47\linewidth]{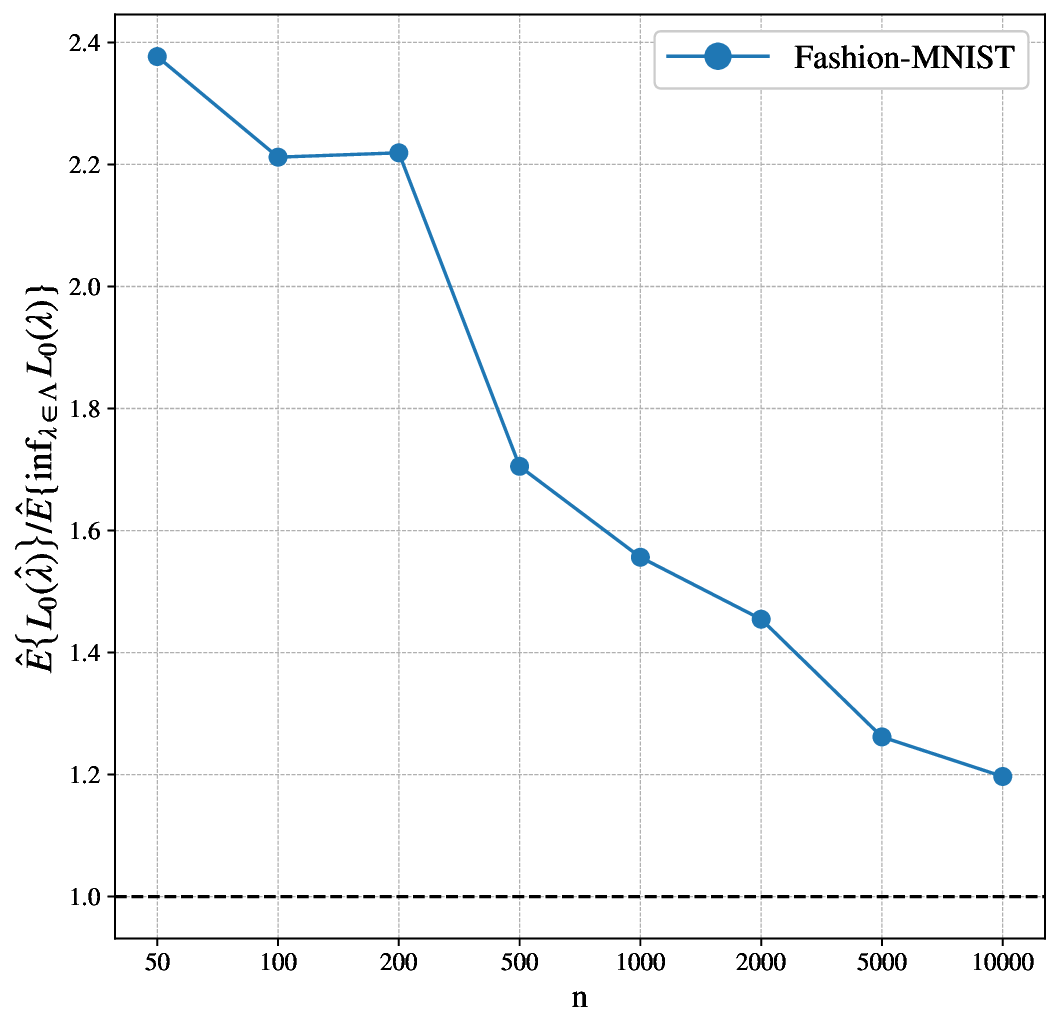}}
    \caption{Results of CNNs for binary classification on Fashion-MNIST. (a) The estimated value of $E\{L_0(\hat{\lambda})/\inf_{\lambda\in\Lambda}L_0(\lambda)\}$. (b) The estimated value of $E\{L_0(\hat{\lambda)}\}/E\{\inf_{\lambda\in\Lambda}L_0(\lambda)\}$.}
    \label{fig:fashion}
\end{figure}

\subsubsection{Binary-Classification on MNIST} \label{sec:4.2.2}
We also consider a binary-Classification problem on the MNIST dataset, which aims at identifying whether a handwritten digit represents a number `1'. Following the procedure and the data generating process introduced in Section \ref{sec:4.2}, we create a new dataset derived from MNIST for binary-classification and train a well-performed neural network model $F_2$ on it. See Appendix \ref{appendix:3} for details of $F_2$.

We consider training a set of CNNs that have the same standard architecture in Section \ref{sec:4.2.1}. The fixed hyperparameters are also set as same as those in Section \ref{sec:4.2.1}, and the varying hyperparamters in $\Lambda$ can be viewed in `MNIST' Table \ref{tab:2}. The results are shown in Figure \ref{fig:4}. It is clear that the ratio converge to 1 in Figure \ref{fig:4a}. The convergence speed of the ratio shown in \ref{fig:4b} is shower than in other scenarios, which might be attributed to simplicity of this problem. Specifically, we observed the performance of models under different hyperparamters and discovered that almost every network achieved satisfactory prediction results, which means a larger sample size is required to effectively distinguish the optimal hyperparameter.

\begin{figure}[htbp]
    \centering
    \subfigure[]{
    \label{fig:4a}
    \includegraphics[width=0.47\linewidth]{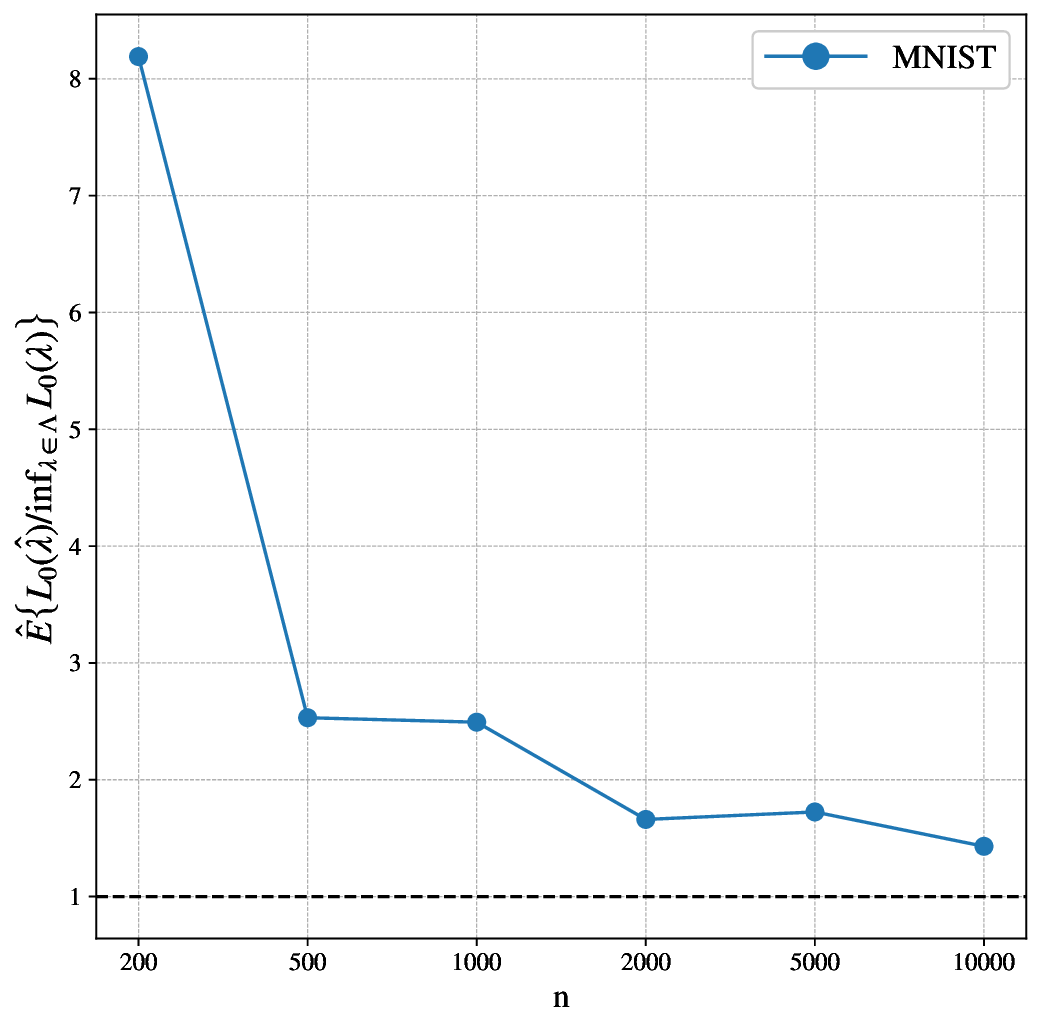}}
    \subfigure[]{
    \label{fig:4b}
    \includegraphics[width=0.47\linewidth]{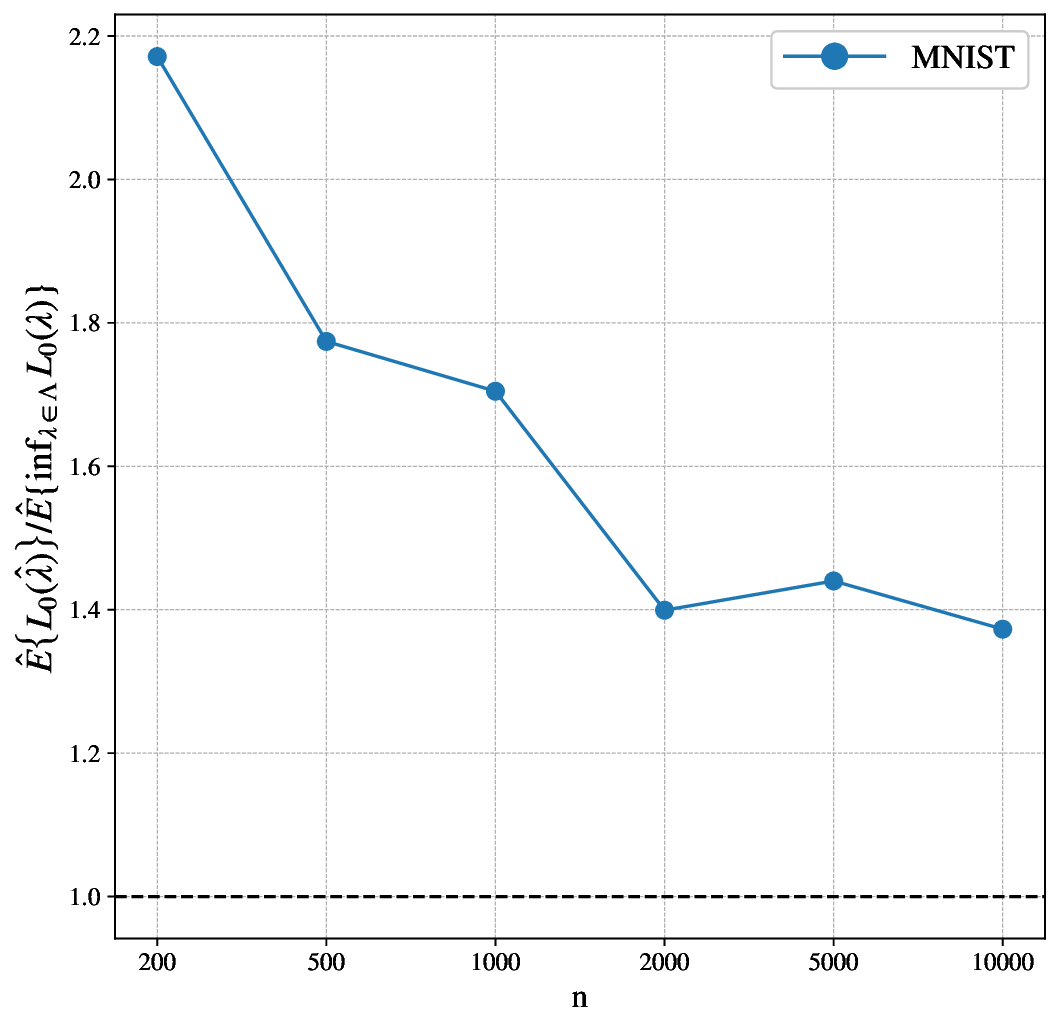}}
    \caption{Results of CNNs for binary classification on MNIST. (a) The estimated value of $E\{L_0(\hat{\lambda})/\inf_{\lambda\in\Lambda}L_0(\lambda)\}$. (b) The estimated value of $E\{L_0(\hat{\lambda)}\}/E\{\inf_{\lambda\in\Lambda}L_0(\lambda)\}$.}
    \label{fig:4}
\end{figure}

\begin{table}[htbp]
    \centering
    \begin{tabular}{lcccc}
        \toprule
        & LR & HS & Depth & L \\
        \midrule
        RNN & 0.01, 0.001 & 50, 100 & 1, 2 & 3, 4, 5, 6 \\
        \bottomrule
    \end{tabular}
    \caption{Design of $\Lambda$ for experiments with RNNs}    \label{tab:3}
\end{table}

\subsection{Recurrent Neural Networks}
In this section, we present an experiment to show that the optimality based on another neural network architecture, recurrent neural network (RNN), holds as well. Since RNN is often used to deal with temporal dynamic behavior of inputs, we designed a time series forecasting problem here. Two types of time series we consider are:
\begin{align}\label{eq:ts}
\begin{aligned}
&\text{1. } y_t = 0.6y_{t-1} + 0.3y_{t-2} - 0.1y_{t-3} +  \epsilon_t,  \\
&\text{2. } y_t=0.3 y_{t-1}+0.6 y_{t-2}+\left(0.1-0.9 y_{t-1}+0.8 y_{t-2}\right)\left[1+\exp \left(-10 y_{t-1}\right)\right]^{-1}+\epsilon_t, 
\end{aligned}    
\end{align}

which are linear and nonlinear respectively, where $\epsilon_t$ is independent for $t\geq 0$ and follows  $N(0,\sigma^2)$ and the initial values of these time series follow uniform distribution between $-6$ to $6$. Our task here is to make one-step forward prediction for $y_t$ given the history values of $\{y_{i}\}_{i=1}^{t-1}$. To construct the samples for training from the origin time series, we apply the sliding window technique. A time series with length $T$ can generate $T-p$ training samples, where $p$ is the input length.  


Apart from the hyperparameters considered in the previous simulation, the input length is considered to be another hyperparameter, which is commonly required to be tuned in practice. See Table \ref{tab:3} for details of $\Lambda$, where $L$ represents the input length. Again we proceed with our simulations with different choice of sample size $n$, and deal with the results similarly, and we show them in Figure \ref{fig:5}. Although the speed of convergence is slower than in those cases in feedforward neural networks, the two concerning ratios are indeed approaching 1.

\begin{figure}[htbp]
    \centering
    \subfigure[]{
    \label{fig:5a}
    \includegraphics[width=0.47\linewidth]{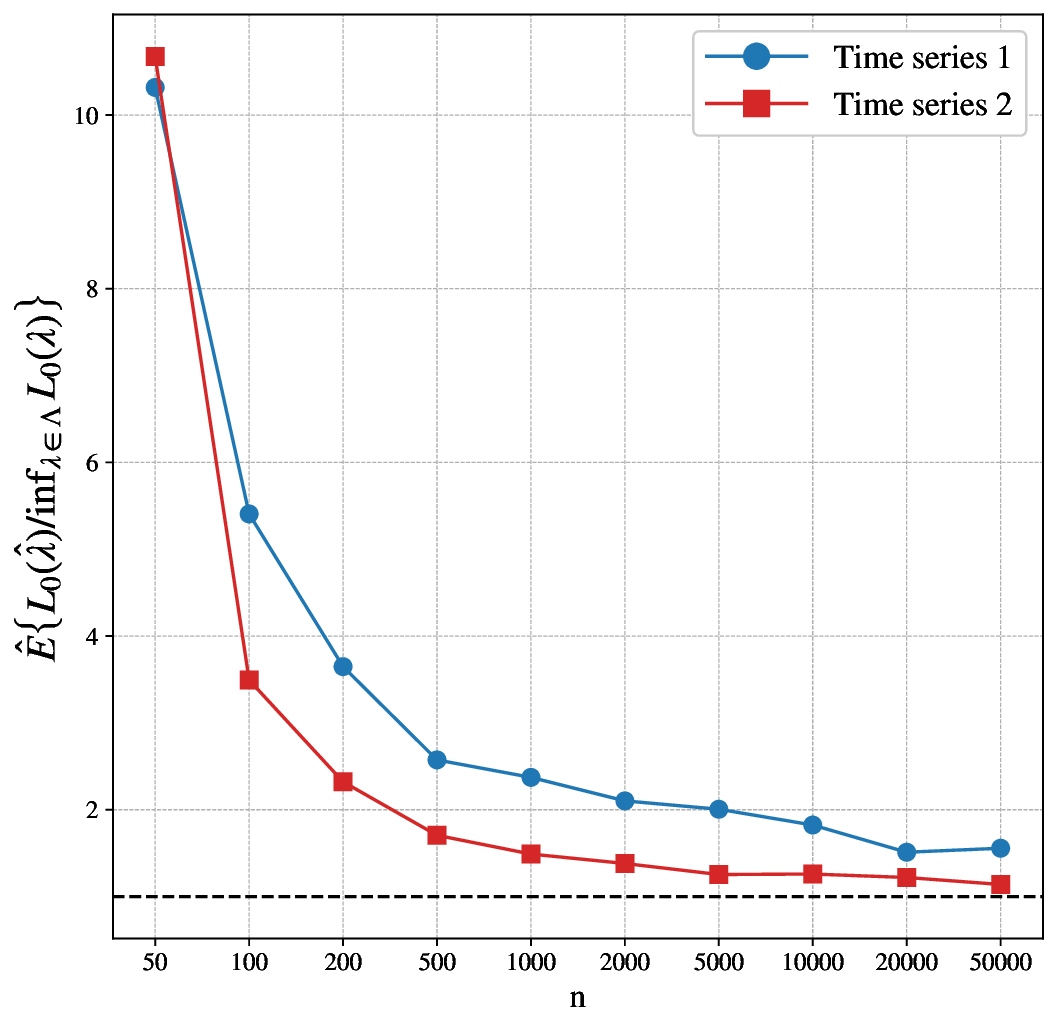}}
    \subfigure[]{
    \label{fig:5b}
    \includegraphics[width=0.47\linewidth]{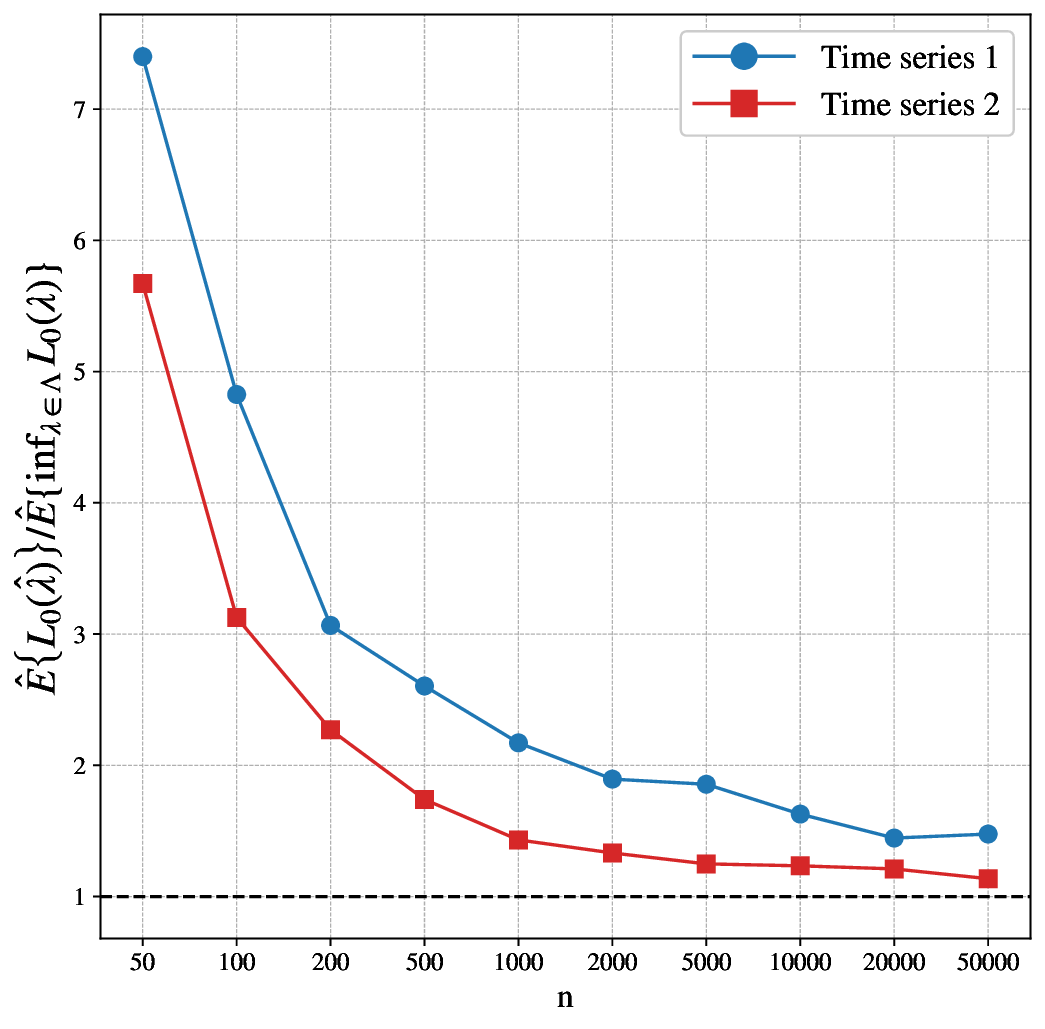}}
    \caption{Results of RNNs for time series forecasting. (a) The estimated value of $E\{L_0(\hat{\lambda})/\inf_{\lambda\in\Lambda}L_0(\lambda)\}$. (b) The estimated value of $E\{L_0(\hat{\lambda)}\}/E\{\inf_{\lambda\in\Lambda}L_0(\lambda)\}$.}
    \label{fig:5}
\end{figure}

Since the samples obtained by (\ref{eq:ts}) and the subsequent sliding window method are not independent, in violation of our theory’s i.i.d. assumption at the beginning of Section \ref{sec:2}, this scenario is distinct from previous ones. Nonetheless, we still regard our theorem valid in the case that samples are correlated according to results in Figure \ref{fig:5}.

\section{Conclusion} \label{sec:5}
Our study offers a new theoretical explanation for the success of neural networks in practical applications. In order to construct an effective neural network model for any given task, optimization should be performed to obtain the most appropriate hyperparameters for the neural network, as they have a decisive impact on the performance of the final network. According to our proposed theories, it is through the use of sample splitting that we can acquire the optimal hyperparameters for the network, thus enabling the attainment of the optimal neural network from a family of neural network models. Through a series of experiments across various problem domains and data distributions, we verify our theoretical results on several prevalent neural network architectures. 

Our framework has contributed new theoretical insights into neural networks, but further research is needed to address the limitations in this paper and explore more widely applicable theories. In the process of establishing Theorem \ref{theorem:2}, our intuition suggested that the condition of the divergence of $n_0$ could be relaxed while proving (\ref{eq:ad1}) and (\ref{eq:ad2}), since both equations involve taking expectations. Nonetheless, we have encountered challenges with a relaxed condition and thus, leave it as a topic for future research. Inspired by our experiment on recurrent neural networks, extending our theory beyond an i.i.d. assumption will also be part of our future work. Finally, while the largest parameter space remained fixed in this article, another future extension would be to consider the size of the available set of hyperparameters diverging with the number of samples, which would also involve the diverging dimension of model parameter $w$.



\newpage
\appendix
\begin{appendix}
\allowdisplaybreaks[4]

\section{Examples for Verifying Assumption \ref{assumption:2}(ii)} \label{appendix:1}
In this section, we provide several examples to illustrate that Assumption \ref{assumption:2}(ii) holds under some primitive conditions. Specifically, we only assume that $E\|X\|^2=O(1)$. Before proceeding, we claim that
\begin{align}
\label{eq:a1.1}
    \sup_{\lambda\in\Lambda}\sup_{w^{0}\in \mathcal{O}(w^*(\lambda),\rho)} \|w^0\|\leq C_{w},
\end{align}
where $C_{w}$ is a fixed constant. In fact, since $|\Lambda|<\infty$, this is straightforward by letting
\begin{align}
    C_{w} = \sup_{\lambda\in\Lambda} \|w^*(\lambda)\| + \rho. \notag
\end{align}
Then, we present two examples as follows, whose conclusions can also be extended to more complicated architectures.

\subsection{Multilayer Perceptrons} \label{appendix:1.1}
Consider a two-layer full-connected neural network corresponding to hyperparameter $\lambda$ with $k$ hidden nodes and input $x=(x_1,...,x_p)^\top\in\mathbb{R}^p$
\begin{align}
&f_{\lambda}(x,w) = \sum_{i=1}^{k(\lambda)} \alpha_ih_i(x)+\alpha_0, \notag\\
&h_i(x)= \sigma(\sum_{j=1}^p \beta_{ij}x_j+b_i),\quad i=1,...,k, \notag 
\end{align}
where $\sigma(\cdot)$ is the sigmoid activation function. Different choices of $\lambda$ correspond to varying values of $k$. Here, $w=(\alpha_0,...,\alpha_k,b_1,...,b_k,\beta_{11},...,\beta_{kp})^\top$. It can be easily shown that $\partial f/\partial\alpha_0=1$, $\partial f/\partial\alpha_i=h_i(x)$, $\partial f/\partial b_i=\alpha_ih_i(x)(1-h_i(x))$, and that $\partial f/\partial\beta_{ij}=\alpha_ix_jh_i(x)(1-h_i(x))$. This implies
$$
\bigg\|\frac{\partial f_{\lambda}(X,w)}{\partial w}\bigg\|^2 \leq c_1+c_2(1+\|X\|^2)\cdot \|w\|^2.
$$
where $c_1$ and $c_2$ are constants that can be chosen to be independent of $k$ due to $|\Lambda|\leq\infty$. Then, by (\ref{eq:a1.1}) and $E\|X\|^2=O(1)$, we obtain
\begin{align}
    E\sup_{w^{0}\in \mathcal{O}(w^*(\lambda),\rho)}\bigg\|\frac{\partial f_{\lambda}(X,w)}{\partial w}\Big|_{w=w^0}\bigg\|\leq E\big\{\sqrt{c_1+c_2(1+\|X\|^2)C_{w}^2 }\big\} = O(1) \notag
\end{align}
This demonstrates that (\ref{eq:assumption2}) in Assumption \ref{assumption:2} holds for this example.

\subsection{Recurrent Neural Networks} \label{appendix:1.2}
We consider a simple recurrent neural network (RNN) structure 
\begin{align}
    &h^{(1)} = \sigma(Ux^{(1)}+b), \notag\\
    &h^{(t)} = \sigma(Wh^{(t-1)}+Ux^{(t)}+b),\quad t\geq2\notag\\
    & o^{(t)} = \tanh(V\cdot h^{(t)}+b^{\prime}), \quad t\geq2\notag,
\end{align}
where $x^{(t)}\in\mathbb{R}^{p}$ denote the input, with $p$ representing the size of the input, and $o^{(t)}\in \mathbb{R}$ denote output at time step $t$. The sigmoid and hyperbolic tangent activation functions are denoted by $\sigma(\cdot)$ and $\tanh(\cdot)$, respectively. Let $w=\text{vec}(U,V,W,b,b^\prime)$ denote the network parameters, where $U\in \mathbb{R}^{K\times p}$, $W\in \mathbb{R}^{K\times K}$, $V\in \mathbb{R}^{K}$, $b\in \mathbb{R}^{K}$, and $b^{\prime}\in \mathbb{R}$.

Now consider a sequential input $x=(x^{(1)},...,x^{(T)})\in\mathbb{R}^{p\times T}$, and assume that the output of this network is given by $f_{\lambda}(w,x) = o^{(T)}$. The different choices of $\lambda$ can correspond to varying values of hyperparameters $T$, $p$ and $K$. To simplify the notation, we define the following terms: 
$$
c^{(t)} = Wh^{(t-1)}+Ux^{(t)}+b,\quad z^{(t)} = Vh^{(t)}+b^{\prime},
$$ 
and we use subscripts to represent the entries of vector or matrix as follows: $U=(U_{ij})_{K\times p}$, $W=(W_{ij})_{K\times K}$, $V=(V_{i})_{K}$, $c^{(t)}=(c^{(t)}_i)_{K}$, and $h^{(t)}= (h^{(t)}_i)_{K}$. To verify Assumption \ref{assumption:2}(ii), we derive some upper bounds for the derivatives of $o^{(T)}$ with respect to $w$.

By performing straightforward differentiation, we obtain
\begin{align}
    &\frac{\partial o^{(T)}}{\partial b^{\prime}} = \frac{\partial o^{(T)}}{\partial z^{(T)}}, \notag\\
    &\frac{\partial o^{(T)}}{\partial V_i} = h_i^{(T)}\frac{\partial o^{(T)}}{\partial z^{(T)}},\quad 1\leq i\leq K. \notag
\end{align}
According to properties of sigmoid and tanh function, for any $t\geq 1$ and any $1\leq j,l\leq K$, we have
\begin{align}
\label{eq:rnn3.3}
    &\big|h_j^{(t)}\big|<1, \quad \big|z^{(t)}\big|<1, \notag \\
    &\bigg|\frac{\partial h^{(t)}_{l}}{\partial c^{(t)}_{l}}\bigg| = \big|\sigma( c^{(t)}_{l})(1-c^{(t)}_{l})\big|<1, \notag\\
    &\bigg|\frac{\partial o^{(t)}}{\partial z^{(t)}}\bigg| = \big|1- \tanh^2(z^{(t)})\big|<1,
\end{align}
which implies that the derivatives of $o^{(T)}$ with respect to $b^\prime$ and $V$ are bounded by constants. For the remaining parameters, we claim that for any $t=2,...,T$, there exists constants $c^{(t)}$ and polynomials $C^{(t)}=C^{(t)}(\|w\|)$ respect to $\|w\|$ with at most degree $t$ and non-negative coefficients, such that 
\begin{align}
    &\bigg| \frac{\partial h^{(t)}_{l}}{\partial b_{i}}\bigg|\leq  C^{(t)},\quad 1\leq i,l \leq K \notag\\
    &\bigg| \frac{\partial h^{(t)}_{l}}{\partial W_{ij}}\bigg|\leq C^{(t)},\quad 1\leq i,j,l \leq \notag K\\
    &\bigg|\frac{\partial h^{(t)}_{l}}{\partial U_{ij}}\bigg|\leq \|x\|\cdot C^{(t)},\quad 1\leq i,l \leq K,\ 1\leq j\leq p \label{eq:rnn1.5}
\end{align}
To illustrate the claim (\ref{eq:rnn1.5}) by mathematical induction, we first consider the case when $t=2$. By applying the chain rule, we have
\begin{align}
    \frac{\partial h^{(2)}_{l}}{\partial b_{i}} &= \frac{\partial h^{(2)}_{l}}{\partial c^{(2)}_{l} }\cdot \frac{\partial c^{(2)}_{l}}{\partial b_{i}} \notag\\
    &=\frac{\partial h^{(2)}_{l}}{\partial c^{(2)}_{l} }\cdot\bigg\{\delta_{il}+W_{li}\frac{\partial h^{(1)}_i}{\partial c^{(1)}_{i}} \bigg\}\label{eq:rnn1.6} \\
    \frac{\partial h^{(2)}_{l}}{\partial W_{ij}} &= \frac{\partial h^{(2)}_{l}}{\partial c^{(2)}_{l}}\cdot \frac{\partial c^{(2)}_{l}}{\partial W_{ij}} \notag\\
    & = \frac{\partial h^{(2)}_{l}}{\partial c^{(2)}_{l}} \cdot \frac{\partial (\sum_{m=1}^K W_{lm}h^{(1)}_m) }{\partial W_{ij}} =h_j^{(1)} \frac{\partial h^{(2)}_{l}}{\partial c^{(2)}_{l}},\label{eq:rnn2}\\
    \frac{\partial h^{(2)}_{l}}{\partial U_{ij}} &= \frac{\partial h^{(2)}_{l}}{\partial c^{(2)}_{l}}\cdot \frac{\partial c^{(2)}_{l}}{\partial U_{ij}} \notag \\
    &=  \frac{\partial h^{(2)}_{l}}{\partial c^{(2)}_{l}}\cdot\bigg\{\frac{\partial \sum_{m=1}^K W_{lm}h^{(1)}_m }{\partial U_{ij} }+ \frac{\partial \sum_{m=1}^K U_{lm}x_m^{(2)}}{\partial U_{ij}}\bigg\} \notag\\
    &= \frac{\partial h^{(2)}_{l}}{\partial c^{(2)}_{l}}\cdot \bigg\{ \sum_{m=1}^KW_{lm}\frac{\partial h^{(1)}_m }{\partial U_{ij}} + \delta_{il}x_j^{(2)} \bigg\} \notag\\
    &= \frac{\partial h^{(2)}_{l}}{\partial c^{(2)}_{l}}\cdot \bigg\{ W_{li}x_j^{(1)}\frac{\partial h_i^{(1)}}{\partial c_i^{(1)}} + \delta_{il}x_j^{(2)} \bigg\},  \label{eq:rnn5}
\end{align}
Then, by (\ref{eq:rnn3.3}), (\ref{eq:rnn1.6}), (\ref{eq:rnn2}) and  (\ref{eq:rnn5}), we can obtain
\begin{align}
\label{eq:rnn1.8}
    &\bigg|\frac{\partial h^{(2)}_{l}}{\partial b_{i}}\bigg| \leq 1+\|W\|,\notag\\
    &\bigg|\frac{\partial h^{(2)}_{l}}{\partial W_{ij}}\bigg| <1, \notag\\
    &\bigg|\frac{\partial h^{(2)}_{l}}{\partial U_{ij}}\bigg| \leq (1+\|W\|)\|x\|,
\end{align}
Therefore, by (\ref{eq:rnn1.8}) and $\|V\|,\|W\|\leq\|w\|$, our claim (\ref{eq:rnn1.5}) holds for the case $t=2$. Now, we assume that (\ref{eq:rnn1.5}) holds for case $t-1$, and consider the case at $t$. By the chain rule, we have
\begin{align}
    \frac{\partial h^{(t)}_{l}}{\partial b_{i}} &= \frac{\partial h^{(t)}_{l}}{\partial c^{(t)}_{l} }\cdot \frac{\partial c^{(t)}_{l}}{\partial b_{i}} \notag\\
    &=\frac{\partial h^{(t)}_{l}}{\partial c^{(t)}_{l} }\cdot\bigg\{\delta_{il}+\sum_{m=1}^K W_{lm}\frac{\partial h^{(t-1)}_m}{\partial b_{i}} \bigg\},\notag\\
    \frac{\partial h^{(t)}_{l}}{\partial W_{ij}} &=  \frac{\partial h^{(t)}_{l}}{\partial c^{(t)}_{l}}\cdot \frac{\partial c^{(t)}_{l}}{\partial W_{ij}} \notag\\
    &= \frac{\partial h^{(t)}_{l}}{\partial c^{(t)}_{l}} \cdot \frac{\partial (\sum_{m=1}^KW_{lm}h^{(t-1)}_m) }{\partial W_{ij}} \notag\\
    &= \frac{\partial h^{(t)}_{l}}{\partial c^{(t)}_{l}} \bigg\{ \sum_{m=1}^K W_{lm}  \frac{\partial h^{(t-1)}}{\partial W_{ij}} +\delta_{il}h_j^{(t-1)}\bigg\},  \notag\\
    \frac{\partial h^{(t)}_{l}}{\partial U_{ij}} &=  \frac{\partial h^{(t)}_{l}}{\partial c^{(t)}_{l}}\cdot \frac{\partial c^{(t)}_{l}}{\partial U_{ij}} \notag\\
    &=\frac{\partial h^{(t)}_{l}}{\partial c^{(t)}_{l}} \cdot \bigg\{  \frac{\partial\sum_{m=1}^KW_{lm}h_m^{(t-1)} }{\partial U_{ij}} + \frac{\partial \sum_{m=1}^K U_{lm}x_m^{(t)}}{\partial U_{ij}} \bigg\} \notag\\
    &= \frac{\partial h^{(t)}_{l}}{\partial c^{(t)}_{l}} \cdot \bigg\{  \sum_{m=1}^KW_{lm}\frac{\partial h_m^{(t-1)}}{\partial U_{ij}} + \delta_{il}x_j^{(t)} \bigg\} \notag
\end{align}
which implies
\begin{align}
\label{eq:rnn7.8}
    &\bigg|\frac{\partial h^{(t)}_{l}}{\partial b_{i}}\bigg| \leq 1+K\|W\|\max_{m\leq K}\bigg|\frac{\partial h^{(t-1)}_{l}}{\partial b_{i}}\bigg| \notag\\
    &\bigg|\frac{\partial h^{(t)}_{l}}{\partial W_{ij}} \bigg| \leq 1+K\|W\|\max_{m\leq K}\bigg|\frac{\partial h^{(t-1)}_{l}}{\partial W_{ij}}\bigg|, \notag\\
    &\bigg|\frac{\partial h^{(t)}_{l}}{\partial U_{ij}} \bigg| \leq \|x\|+K\|W\|\max_{m\leq K}\bigg|\frac{\partial h^{(t-1)}_{l}}{\partial U_{ij}} \bigg|.
\end{align}
By using (\ref{eq:rnn7.8}) and the induction hypothesis, we can construct an appropriate $C^{(t)}$ for (\ref{eq:rnn1.5}) for the case $t$, thereby completing the induction. 

Now, we proceed to upper bound our objectives. By applying the chain rule, we have
\begin{align}
    \frac{\partial o^{(t)}}{\partial b_{i}} &=\frac{\partial o^{(t)}}{\partial z^{(t)} }\sum_{l=1}^KV_{l}\cdot\frac{\partial h^{(t)}_{l}}{\partial b_{i}},\notag\\
    \frac{\partial o^{(t)}}{\partial U_{ij}} &= \frac{\partial o^{(t)}}{\partial z^{(t)}}\sum_{l=1}^KV_{l}\cdot\frac{\partial h^{(t)}_{l}}{\partial U_{ij}},\notag\\
    \frac{\partial o^{(t)}}{\partial W_{ij}} &= \frac{\partial o^{(t)}}{\partial z^{(t)}}\sum_{l=1}^K V_{l}\cdot \frac{\partial h^{(t)}_{l}}{\partial W_{ij}}. \label{eq:rnn10}
\end{align}
Then, it follows from (\ref{eq:rnn3.3}), (\ref{eq:rnn7.8}), (\ref{eq:rnn10}) and Cauchy-Schwarz inequality that 
\begin{align}
    &\bigg|\frac{\partial o^{(T)}}{\partial b_{i}}\bigg| \leq \sqrt{K}\|V\|\bigg|\frac{\partial h^{(t)}_{l}}{\partial b_{i}}\bigg|,\notag\\
    &\bigg|\frac{\partial o^{(T)}}{\partial W_{ij}} \bigg| \leq \sqrt{K}\|V\|\bigg|\frac{\partial h^{(t)}_{l}}{\partial W_{ij}}\bigg|,\notag\\
    &\bigg|\frac{\partial o^{(T)}}{\partial U_{ij}} \bigg| \leq \sqrt{K}\|V\|\bigg|\frac{\partial h^{(t)}_{l}}{\partial U_{ij}} \bigg|. \label{eq:rnn11}
\end{align}
Incorporating (\ref{eq:rnn1.5}) and (\ref{eq:rnn11}), we obtain that
\begin{align}
\label{eq:new-0}
    \bigg\|\frac{\partial f_{\lambda}(x,w)}{\partial w}\bigg\| = \bigg\|\frac{\partial o^{(T)}}{\partial w}\bigg\| &\leq \sqrt{c_1+ c_2{\{C^{(T)}(\|w\|)\}}^2+c_3{\{C^{(T)}(\|w\|)\}}^2\|x\|^2} \notag\\
    &\leq c_1+ c_2{\{C^{(T)}(\|w\|)\}}^2+c_3{\{C^{(T)}(\|w\|)\}}^2\|x\|^2
\end{align}
where $c_1$ ($c_1>1$), $c_2$ and $c_3$ are constants that, according to $|\Lambda|\leq \infty$, can be selected independently of any $K$, $p$, and $T$. Finally, since $C^{(T)}$ is monotonically increasing with respect to $\|w\|$, it follows from (\ref{eq:a1.1}) and (\ref{eq:new-0})that 
\begin{align}
    E\sup_{w^{0}\in \mathcal{O}(w^*(\lambda),\rho)}\bigg\|\frac{\partial f_{\lambda}(X,w)}{\partial w}\Big|_{w=w^0}\bigg\|&\leq E\Big\{c_1+c_2{\{C^{(T)}(C_w)\}}^2+c_3{\{C^{(T)}(C_w)\}}^2\|X\|^2 \Big\}\notag\\
    &= O(1), \notag
\end{align}
where the equality follows from our assumption $E\|X\|^2=O(1)$. This verify Assumption \ref{assumption:2}(ii) in this example.

\section{Proof of Theoretical Results}  \label{appendix:2}
\subsection{Supporting Lemma(s)} \label{appendix:2.1}
We introduce the following lemma that will be used in our proof. 

\begin{lemma} \label{lemma:1}
Let $\Lambda$ be the available set of hyperparameters, which can have finite or diverging elements. For any $\lambda\in\Lambda$, $A_n(\lambda)$, $a_n(\lambda)$, and $B_n(\lambda)$ are sequences related to $\lambda$. Let $\eta_n$ be a real-valued positive sequence that satisfies $\eta_n=o(1)$. Define $\hat{\lambda} = \operatorname{argmin}_{\lambda\in\Lambda}\{A_n(\lambda)+a_n(\lambda)\}$. If 
\begin{align} \label{eq:lem1}
\sup_{\lambda\in\Lambda} \bigg|\frac{a_n(\lambda)}{B_n(\lambda)} \bigg| = O_p(\eta_n),
\end{align}
\begin{align}\label{eq:lem2}
\sup_{\lambda\in\Lambda} \bigg| \frac{A_n(\lambda)-B_n(\lambda)}{B_n(\lambda)}\bigg|  = O_p(\eta_n),
\end{align}
and there exists a interger $N$ and a constant $\kappa$ such that when $n\geq N$,
\begin{align} \label{eq:lem3}
    \inf_{\lambda \in \Lambda} B_n(\lambda) \geq \kappa>0,
\end{align}
then 
\begin{align}\label{eq:lem4}
   \frac{A_n(\hat{\lambda})}{\operatorname{inf}_{\lambda \in \Lambda}A_n(\lambda)} = 1+O_p(\eta_n) =1+o_p(1).
\end{align}

\end{lemma}
\begin{proof}
By the definition of infimum, there exist a sequence $\lambda(n)\in\Lambda$ and a non-negative sequence $\vartheta_n(\lambda(n))=o_p(1)$ such that
\begin{align} \label{eq:lem5}
    \inf_{\lambda\in\Lambda} A_n(\lambda) = A_n(\lambda(n))-\vartheta_n(\lambda(n)).
\end{align}
In addition, it follows from (\ref{eq:lem2}) that
\begin{align} \label{eq:lem6}
    \inf_{\lambda\in\Lambda} \frac{A_n(\lambda)}{B_n(\lambda)} &= \inf_{\lambda\in\Lambda}\bigg(\frac{A_n(\lambda)}{B_n(\lambda)}-1\bigg)+1 \notag \\
    &\geq -\sup_{\lambda\in\Lambda} \bigg|\frac{A_n(\lambda)}{B_n(\lambda)}-1\bigg| + 1  = 1 - O_p(\eta_n).
\end{align}
By (\ref{eq:lem3}) and (\ref{eq:lem6}), we have
\begin{align} \label{eq:lem7}
    \inf_{\lambda\in\Lambda} \frac{|A_n(\lambda)-\vartheta_n(\lambda(n))|}{B_n(\lambda)} & \geq \inf_{\lambda\in\Lambda} \frac{A_n(\lambda)-\vartheta_n(\lambda(n))} {B_n(\lambda)} \geq \inf_{\lambda\in\Lambda}\frac{A_n(\lambda)}{B_n(\lambda)} - \frac{\vartheta_n(\lambda(n))}{\inf_{\lambda\in\Lambda}B_n(\lambda)}\notag\\
    & \geq -\sup_{\lambda\in\Lambda} \bigg|\frac{A_n(\lambda)}{B_n(\lambda)}-1\bigg| + 1 - \frac{\vartheta_n(\lambda(n))}{\inf_{\lambda\in\Lambda}B_n(\lambda)}\notag\\
    & \geq 1 - O_p(\eta_n) - \frac{\vartheta_n(\lambda(n))}{\kappa}.
\end{align}
Then, we have that
\begin{align}
    \bigg|\frac{\inf_{\lambda\in\Lambda}A_n(\lambda)}{A_n(\hat{\lambda})}-1  \bigg| &= \frac{A_n(\hat{\lambda})-\inf_{\lambda\in\Lambda}A_n(\lambda)}{A_n(\hat{\lambda})} \notag\\
    & =  \frac{ \inf_{\lambda\in\Lambda}(A_n(\lambda)+a_n(\lambda))-a_n(\hat{\lambda})-\inf_{\lambda\in\Lambda}A_n(\lambda) }{A_n(\hat{\lambda}) } \notag\\
    & \leq \frac{A_n(\lambda(n))+a_n(\lambda(n))-a_n(\hat{\lambda})-A_n(\lambda(n))+\vartheta_n}{A_n(\hat{\lambda})} \notag \\
    & \leq \frac{|a_n(\lambda(n))|}{A_n(\hat{\lambda})}+\frac{|a_n(\hat{\lambda})|}{A_n(\hat{\lambda})} + \frac{\vartheta_n}{A_n(\hat{\lambda})} \notag \\
    & \leq \frac{|a_n(\lambda(n))|}{\inf_{\lambda\in\Lambda}A_n(\lambda)} + \frac{|a_n(\hat{\lambda})|}{A_n(\hat{\lambda})} + \frac{\vartheta_n}{A_n(\hat{\lambda})} \notag\\
    & \leq \frac{|a_n(\lambda(n))|}{A_n(\lambda(n))-\vartheta_n} + \frac{|a_n(\hat{\lambda})|}{A_n(\hat{\lambda})} + \frac{\vartheta_n}{A_n(\hat{\lambda})} \notag \\
    & \leq \sup_{\lambda\in\Lambda}\frac{|a_n(\lambda)|}{B_n(\lambda)}\sup_{\lambda\in\Lambda}\frac{B_n(\lambda)}{|A_n(\lambda)-\vartheta_n|} + \sup_{\lambda\in\Lambda}\frac{|a_n(\lambda)|}{B_n(\lambda)}\sup_{\lambda\in\Lambda}\frac{B_n(\lambda)}{A_n(\lambda)}\notag\\
    &\qquad\qquad +\sup_{\lambda\in\Lambda}\frac{\vartheta_n}{B_n(\lambda)}\sup_{\lambda\in\Lambda}\frac{B_n(\lambda)}{A_n(\lambda)} \notag\\
    & \leq \sup_{\lambda\in\Lambda}\frac{|a_n(\lambda)|}{B_n(\lambda)}\bigg[\inf_{\lambda\in\Lambda}\frac{|A_n(\lambda)-\vartheta_n|}{B_n(\lambda)} \bigg]^{-1} + \sup_{\lambda\in\Lambda}\frac{|a_n(\lambda)|}{B_n(\lambda)} \bigg[\inf_{\lambda\in\Lambda}\frac{A_n(\lambda)}{B_n(\lambda)}\bigg]^{-1} \notag\\
    & \qquad \qquad +\frac{\vartheta_n}{\inf_{\lambda\in\Lambda}B_n(\lambda)}\bigg[\inf_{\lambda\in\Lambda}\frac{A_n(\lambda)}{B_n(\lambda)}\bigg]^{-1},
\end{align}
where the first equality follows from the definition of $\hat{\lambda}$, the first and third inequality follows from (\ref{eq:lem5}). By (\ref{eq:lem1}), (\ref{eq:lem6}), (\ref{eq:lem7}), and $\vartheta_n(\lambda(n))=o_p(1)$,  we obtain that 
\begin{align}
    \frac{\inf_{\lambda\in\Lambda}A_n(\lambda)}{A_n(\hat{\lambda})}\leq \frac{O_p(\eta_n)}{1-O_p(\eta_n)} + \frac{O_p(\eta_n)}{1-O_p(\eta_n)-o_p(1)/\kappa} + \frac{o_p(1)/\kappa}{1-O_p(\eta_n)}. 
\end{align}
Since $\eta_n = o(1)$, we obtain $\inf_{\lambda\in\Lambda}A_n(\lambda)/A_n(\hat{\lambda})=1+O_p(\eta_n)$ based on straightforward derivation and the definition of $O_p$, which also implies $A_n(\hat{\lambda})/\inf_{\lambda\in\Lambda}A_n(\lambda)=1+O_p(\eta_n)$. Therefore, we complete the proof of Lemma \ref{lemma:1}.
\end{proof}

\subsection{Proof of Theorem \ref{theorem:1}} \label{appendix:2.2}
\begin{proof}
Define $L_n^*(\lambda) = L_n(\lambda) - 
\frac{1}{n}\sum_{i=1}^n\{Y_i-E(Y_i|X_i)\}\{Y_i+E(Y_i|X_i)\}$. We have 
\begin{align}
\hat{\lambda}&=\underset{\lambda\in\Lambda}{\operatorname{argmin}} L_n(\lambda) = \underset{\lambda\in\Lambda}{\operatorname{argmin}} L_n^*(\lambda) \notag\\
&= \underset{\lambda\in\Lambda}{\operatorname{argmin}}\{R_0(\lambda)+L_n^*(\lambda)-R_0(\lambda)\}. \notag
\end{align}
By applying Lemma \ref{lemma:1} with $A_n(\lambda)$, $B_n(\lambda)$, $a_n(\lambda)$, and $\eta_n$ being $R_0(\lambda)$, $R_0^*(\lambda)$, $L_n^*(\lambda)-R_0(\lambda)$, and $\xi_n^{-1}n^{-1/2}$, respectively, it suffices to prove
\begin{align}
    \label{eq:11}\sup_{\lambda\in\Lambda} \frac{|R_0(\lambda)-R^*_0(\lambda)|}{|R^*_0(\lambda)|}=  O(\xi_n^{-1}n^{-1/2})   =  o(1)
\end{align}
and
\begin{align}
    \label{eq:ex}\sup_{\lambda\in\Lambda}\frac{|L_n^*(\lambda)-R_0(\lambda)|}{|R^*_0(\lambda)|} =  O_p(\xi_n^{-1}n^{-1/2})  = o_p(1),
\end{align}
to establish Theorem \ref{theorem:1}. Observe that if (\ref{eq:11}) holds, the following equation 
\begin{align}
    \label{eq:12}\sup_{\lambda\in\Lambda}\frac{|L_n^*(\lambda)-R^*_0(\lambda)|}{|R^*_0(\lambda)|} = O_p(\xi_n^{-1}n^{-1/2})  = o_p(1).
\end{align}
is sufficient to prove (\ref{eq:ex}). Therefore, we aim to prove (\ref{eq:11}) and (\ref{eq:12}).

By Assumption \ref{assumption:1} and Assumption \ref{assumption:2}(ii), we obtain
\begin{align}
\label{eq:13}
    \big|f_{\lambda}(X_0,\hat{w}_n(\lambda))-f_{\lambda}(X_0,w^*(\lambda))\big| &=\bigg|(\hat{w}_n(\lambda)-w^*(\lambda))^{\top}\frac{\partial f_{\lambda}(X_0,w)}{\partial w}\Big|_{w=w^{0}}\bigg|\notag\\
    &\leq \|\hat{w}_n(\lambda)-w^*(\lambda)\|\times \bigg\|\frac{\partial f_{\lambda}(X_0,w)}{\partial w}|_{w=w^{0}}\bigg\| \notag\\
    & = O_p(n^{-1/2})\times O_p(1) = O_p(n^{-1/2}).
\end{align}
Therefore,
\begin{align}
\label{eq:14}
    &\xi_{n}^{-1} \sup_{\lambda\in\Lambda}\big|\{f_{\lambda}(X_0,\hat{w}_n(\lambda))-E(Y_0|X_0)\}^2 -\{f_{\lambda}(X_0,w^*(\lambda))-E(Y_0|X_0)\}^2\big|\notag\\
    =\ &\xi_{n}^{-1} \sup_{\lambda\in\Lambda}\Big|\big\{f_{\lambda}(X_0,\hat{w}_n(\lambda))-f_{\lambda}(X_0,w^*(\lambda))\big\}\Big[\{f_{\lambda}(X_0,\hat{w}_n(\lambda))-f_{\lambda}(X_0,w^*(\lambda))\} \notag\\
    &\qquad\qquad+2\{f_{\lambda}(X_0,w^*(\lambda))-E(Y_0|X_0)\}\Big]\Big|\notag\\
    \leq\ &\xi_{n}^{-1}\sup_{\lambda\in\Lambda} \bigg| \|\hat{w}_n(\lambda)-w^*(\lambda)\|\times \big\|\frac{\partial f_{\lambda}(X_0,w)}{\partial w}|_{w=w^{0}}\big\|\notag \\
    &\qquad \qquad  \times \Big[\|\hat{w}_n(\lambda)-w^*(\lambda)\|\times\big\|\frac{\partial  f_{\lambda}(X_0,w)}{\partial w}|_{w=w^{0}}\big\|+2\{f_{\lambda}(X_0,w^*(\lambda))-E(Y_0|X_0)\}\Big] \bigg|\notag\\
    =\ &\xi_{n}^{-1}\times [ O_p(n^{-1/2}) \times \{ O_p(n^{-1/2}) +O_p(1)\}]\notag\\
    =\ &\xi_{n}^{-1} O_p(n^{-1/2}) = o_p(1),
\end{align}
where the second inequality follows from Mean value theorem and the third equality follows from Assumption \ref{assumption:2}(i) and (\ref{eq:13}). Consequently, we have
\begin{align}
    & \sup_{\lambda\in\Lambda} \frac{|R_0(\lambda)-R^*_0(\lambda)|}{|R^*_0(\lambda)|} \notag \\
    \leq\ & \xi_n^{-1} \sup_{\lambda\in\Lambda}|R_0(\lambda)-R^*_0(\lambda)| \notag \\
    =\ & \xi_{n}^{-1} \sup_{\lambda\in\Lambda}\Big|E\big[\{f_{\lambda}(X_0,\hat{w}_n(\lambda))-E(Y_0|X_0)\}^2 -\{f_{\lambda}(X_0,w^*(\lambda))-E(Y_0|X_0)\}^2 \big]\Big|\notag\\
    \leq\ &E\Big[ \xi_{n}^{-1}\cdot \sup_{\lambda\in\Lambda}\big|\{f_{\lambda}(X_0,\hat{w}_n(\lambda))-E(Y_0|X_0)\}^2 -\{f_{\lambda}(X_0,w^*(\lambda))-E(Y_0|X_0)\}^2\big|  \Big]\notag\\
     =\ & O(\xi_n^{-1}n^{-1/2}) =  o(1), \notag
\end{align}
where the third inequality follows from Assumption \ref{assumption:3} and Jensen inequality and the forth equality follows from (\ref{eq:14}). This complete the proof of (\ref{eq:11}). 

Now we present the proof of (\ref{eq:12}). Firstly, we have
\begin{align}
\label{eq:16}
    \sup_{\lambda\in\Lambda}\frac{|L_n^*(\lambda)-R^*_0(\lambda)|}{|R^*_0(\lambda)|} \leq  \xi_n^{-1}\sup_{\lambda\in\Lambda}|L_n^*(\lambda)-R^*_0(\lambda)|.
\end{align}
Then, we examine the supremum on the right hand side of (\ref{eq:16}):
\begin{align}
\label{eq:17}
    &|L_n^*(\lambda)-R^*_0(\lambda)|\notag\\
    =\ & \bigg|\frac{1}{n}\sum_{i=1}^n\big[\{f_{\lambda}(X_i,\hat{w}_n(\lambda))-Y_i\}^2 - \{Y_i-E(Y_i|X_i)\}\{Y_i+E(Y_i|X_i)\}\big] - R^*_0(\lambda) \bigg|\notag\\
    \leq\ &\underbrace{\bigg|\frac{1}{n}\sum_{i=1}^n\big[\{f_{\lambda}(X_i,w^*(\lambda))-Y_i\}^2 -\{Y_i-E(Y_i|X_i)\}\{Y_i+E(Y_i|X_i)\}\big]-R^*_0(\lambda)\bigg|}_{(I)} \notag\\
    &\qquad \qquad +\underbrace{\bigg|\frac{1}{n}\sum_{i=1}^n\big[\{f_{\lambda}(X_i,\hat{w}_n(\lambda))-Y_i\}^2 - \{f_{\lambda}(X_i,w^*(\lambda))-Y_i\}^2\big] \bigg|}_{(II)},
\end{align}
and we address $(I)$ and $(II)$ in (\ref{eq:17}) separately. 

Regrading $(I)$, we have
\begin{align}
\label{eq:18}
    (I) &= \bigg|\frac{1}{n}\sum_{i=1}^n\big[\{f_{\lambda}(X_i,w^*(\lambda))-E(Y_i|X_i) + E(Y_i|X_i) -Y_i\}^2 \notag \\
    &\qquad\qquad\qquad - \{Y_i-E(Y_i|X_i)\}\{Y_i+E(Y_i|X_i)\}\big] - R^*_0(\lambda) \bigg|\notag\\
    & = \bigg|\frac{1}{n}\sum_{i=1}^n\big[\{f_{\lambda}(X_i,w^*(\lambda))-E(Y_i|X_i)\}^2 +2f_\lambda(X_i,w^*(\lambda))\{E(Y_i|X_i)-Y_i\}\big]
     -R^*_0(\lambda)\bigg|\notag\\
    & \leq \frac{1}{n}\Big|\sum_{i=1}^n\{f_{\lambda}(X_i,w^*(\lambda))-E(Y_i|X_i)\}^2 - nR^*_0(\lambda)\Big| +\frac{2}{n}\Big|\sum_{i=1}^n f_\lambda(X_i,w^*(\lambda))\{Y_i-E(Y_i|X_i)\}\Big|.
\end{align}
Define 
$$
S_n^{\prime}(\lambda)=\sum_{i=1}^n\{f_{\lambda}(X_i,w^*(\lambda))-E(Y_i|X_i)\}^2
$$
and
$$ 
S_n^{\prime\prime}(\lambda)=\sum_{i=1}^n f_\lambda(X_i,w^*(\lambda))\{Y_i-E(Y_i|X_i)\}.
$$
Since the samples are i.i.d., it holds that 
\begin{align}
\label{eq:19}
    E(S_n^\prime(\lambda))= \sum_{i=1}^nE(f_\lambda(X_i,w^*(\lambda))-E(Y_i|X_i))^2 = nR^*_0(\lambda).
\end{align}
It follows from (\ref{eq:18}) and (\ref{eq:19}) that 
\begin{align}
\label{eq:20}
    (I) &\leq \frac{1}{n}\big|S_n^{\prime}(\lambda)-E(S_n^{\prime}(\lambda))\big|+\frac{2}{n}|S_n^{\prime\prime}(\lambda)|.
\end{align}
Regarding first term on the right hand side of (\ref{eq:20}), for any $n$,  $\lambda^\prime\in\Lambda$, and $\delta>0$, we have
\begin{align} \label{eqnew:1}
    \Pr\Big\{\Big|n^{-1}\{S_n^{\prime}(\lambda^\prime)-E(S_n^{\prime}(\lambda^\prime))\}\Big|>\delta \Big\} &=\Pr\big\{|S_n^{\prime}(\lambda^\prime)-E(S_n^{\prime}(\lambda^\prime))|>n\delta  \big\} \notag\\
    &\leq\ n^{-2}\delta^{-2} \operatorname{var}(S_n^{\prime}(\lambda^\prime)) 
\end{align}
Then, for any $\delta^{*}>0$ and any $n$, we set $\delta = \delta_n= n^{-1}(\delta^*)^{-1} \{\operatorname{var}(S_n^{\prime}(\lambda^\prime))\}^{1/2} $, plug it into (\ref{eqnew:1}), and find it holds for any $n$ that 
\begin{align} \label{eqnew:2}
    \Pr\Big\{\big|n^{-1}\{S_n^{\prime}(\lambda^\prime)-E(S_n^{\prime}(\lambda^\prime))\}\big|> n^{-1}(\delta^{*})^{-1/2}\{\operatorname{var}(S_n^{\prime}(\lambda^\prime))\}^{1/2} \Big\} \leq \delta^{*},
\end{align}
which implies that
\begin{align}
n^{-1}|S_n^{\prime}(\lambda^\prime)-E(S_n^{\prime}(\lambda^\prime))|&=O_p(n^{-1}\{\operatorname{var}(S_n^{\prime}(\lambda^\prime))\}^{1/2})     \notag\\
& = O_p\Big(n^{-1/2}\big\{\operatorname{var}[\{f_{\lambda^\prime}(X_1,w^*(\lambda^\prime))-E(Y_1|X_1)\}^2]\big\}^{1/2}\Big)\notag\\
& = O_p(n^{-1/2}) \notag
\end{align}
for any $\lambda^\prime\in\Lambda$. Given that $\Lambda$ is a finite set, we obtain
\begin{align}
\label{eq:22}
    \xi_n^{-1}\sup_{\lambda\in\Lambda}\Big|\frac{1}{n}\{S_n^{\prime}(\lambda)-E(S_n^{\prime}(\lambda))\}\Big| = O_p(\xi_n^{-1}n^{-1/2}). 
\end{align}
For second term on the right hand side of (\ref{eq:20}), due to the property of conditional expectation, we have
\begin{align}
    E(S_n^{\prime\prime}(\lambda))&= \sum_{i=1}^n E\Big\{E\big[f_\lambda(X_i,w^*(\lambda))\{Y_i-E(Y_i|X_i)\}|X_i\big]\Big\}\notag\\
    &= \sum_{i=1}^n E\Big\{f_\lambda(X_i,w^*(\lambda))E\{Y_i-E(Y_i|X_i)|X_i\}\Big\} = 0. \notag
\end{align}
For any $\lambda^{\prime\prime}\in\Lambda$ and $\delta>0$, employing Chebyshev inequality, we have
\begin{align}
    \Pr\Big\{\big|2n^{-1}S_n^{\prime\prime}(\lambda^{\prime\prime})\big|>\delta\Big\} &=\Pr\big\{|S_n^{\prime\prime}(\lambda^{\prime\prime})|>n\delta/2 \big\} \notag\\
    &\leq 4n^{-2}\delta^{-2}\operatorname{var}(S_n^{\prime\prime}(\lambda^{\prime\prime})).
\end{align}
Similar to (\ref{eqnew:2}), this implies 
\begin{align}
    n^{-1}S_n^{\prime\prime}(\lambda^{\prime\prime}) &= O_p(n^{-1}\{\operatorname{var}(S_n^{\prime\prime}(\lambda^{\prime\prime}))\}^{1/2} ) \notag\\
    &= O_p(n^{-1/2}\{\operatorname{var}\big( f_\lambda(X_1,w^*(\lambda^{\prime\prime}))\{Y_1-E(Y_1|X_1)\}\big)\}^{1/2} )\notag\\
    &= O_p(n^{-1/2}) \notag
\end{align}
and 
\begin{align}
    \label{eq:25}
    \xi_n^{-1}\sup_{\lambda\in\Lambda}\Big|n^{-1}S_n^{\prime\prime}(\lambda)\Big| = O_p(\xi^{-1}n^{-1/2})
\end{align}
By combining (\ref{eq:20}), (\ref{eq:22}), and (\ref{eq:25}), we achieve that
\begin{align}
\label{eq:26}
  \xi_n^{-1}\sup_{\lambda\in\Lambda} (I) \leq\xi_n^{-1}\sup_{\lambda\in\Lambda}\Big|n^{-1}\{S_n^{\prime}(\lambda)-E(S_n^{\prime}(\lambda))\}\Big|+ \xi_n^{-1}\sup_{\lambda\in\Lambda}\Big|2n^{-1}S_n^{\prime\prime}(\lambda)\Big|= O_p(\xi^{-1}n^{-1/2}).
\end{align}

Now we proceed to handle $(II)$. Apply a technique similar to that in (\ref{eq:14}), we have
\begin{align}
    \sup_{\lambda\in\Lambda} (II) &=  \sup_{\lambda\in\Lambda}\frac{1}{n}\Big|\sum_{i=1}^n\big[\{f_{\lambda}(X_i,\hat{w}_n(\lambda))-Y_i\}^2-\{f_{\lambda}(X_i,w^*(\lambda))-Y_i\}^2 \big] \Big|\notag\\
    &= \sup_{\lambda\in\Lambda}\frac{1}{n}\Big|\sum_{i=1}^n\big\{f_{\lambda}(X_i,\hat{w}_n(\lambda))-f_{\lambda}(X_i,w^*(\lambda)) \big\}\notag\\
    & \qquad\qquad\qquad \times \Big[\big\{f_{\lambda}(X_i,\hat{w}_n(\lambda))-f_{\lambda}(X_i,w^*(\lambda))\big\}+2\{f_\lambda(X_i,w^*(\lambda))-Y_i\} \Big] \Big|\notag\\
    &= O_p(n^{-1})+O_p(n^{-1/2}) = O_p(n^{-1/2}), \notag
\end{align}
where the second equality follows from (\ref{eq:13}) and the last equality follows from Assumption \ref{assumption:2}(i). Then, by Asssumption \ref{assumption:4}, we obtain that 
\begin{align}
\label{eq:28}
    \xi_n^{-1} \sup_{\lambda\in\Lambda} (II) =  O_p(\xi_n^{-1}n^{-1/2}) = o_p(1).
\end{align}

Finally, by combining (\ref{eq:16}), (\ref{eq:17}), (\ref{eq:26}), and (\ref{eq:28}), we have
\begin{align}
  \sup_{\lambda\in\Lambda}\frac{|L_n^*(\lambda)-R^*_0(\lambda)|}{|R^*_0(\lambda)|}& \leq  \xi_n^{-1} \sup_{\lambda\in\Lambda} \{(I)+(II)\}\notag\\
  & \leq \xi_n^{-1} \sup_{\lambda\in\Lambda} (I) +\xi_n^{-1} \sup_{\lambda\in\Lambda} (II)\notag\\
  &=  O_p(\xi_n^{-1}n^{-1/2}) = o_p(1), \notag
\end{align}
which demonstrates that (\ref{eq:12}) holds. Therefore, we complete the proof of Theorem \ref{theorem:1}.
\end{proof}

\subsection{Proof of Theorem \ref{theorem:2}} \label{appendix:2.3}
\begin{proof}
We consider (\ref{eq:7}), the first justification in Theorem \ref{theorem:2}. Rewrite the definition of $\hat{\lambda}$ as
\begin{align}
    \hat{\lambda} = \underset{\lambda\in\Lambda}{\operatorname{argmin}} \left\{L_0(\lambda) + (L_n^*(\lambda) - L_0(\lambda)) \right\}. \notag
\end{align}
Similar to the proof of Theorem \ref{theorem:1}, applying Lemma \ref{lemma:1} with $A_n(\lambda)$, $B_n(\lambda)$, $a_n(\lambda)$, and $\eta_n$ being $L_0(\lambda)$, $R_0^*(\lambda)$, $L_n^*(\lambda)-L_0(\lambda)$ and $\xi_n^{-1}n^{-1/2}+\xi_n^{-1}n_{0}^{-1/2}$, it suffices to prove
\begin{align}
\label{eq:32}
    \sup_{\lambda\in\Lambda} \frac{|R^*_0(\lambda)-L_0(\lambda)|}{|R^*_0(\lambda)|}= O_p(\xi_n^{-1}n^{-1/2} + \xi_n^{-1}n_0^{-1/2})=o_p(1)
\end{align}
and
\begin{align}
    \label{eq:ex2}
    \sup_{\lambda\in\Lambda} \frac{|L^*_n(\lambda)-L_0(\lambda)|}{|R^*_0(\lambda)|}= O_p(\xi_n^{-1}n^{-1/2}+\xi_n^{-1}n_0^{-1/2})=o_p(1).
\end{align}
Notice that $L^*_n(\lambda)-L_0(\lambda)$ can be rewritten to $(L^*_n(\lambda)-R_0^*(\lambda)) + (R_0^*(\lambda) - L_0(\lambda))$. Besides,
according to the conclusion of Theorem \ref{theorem:1}, we have
\begin{align}
\label{eq:ex3}
    \sup_{\lambda\in\Lambda}\frac{|L_n^*(\lambda)-R^*_0(\lambda)|}{|R^*_0(\lambda)|} =  O_p(\xi_n^{-1}n^{-1/2}) =o_p(1),
\end{align}
We claim that (\ref{eq:32}) is sufficient to prove (\ref{eq:ex2}). In fact, the sufficiency is easily obtained by 
\begin{align}
    \sup_{\lambda\in\Lambda}\frac{|L_n^*(\lambda)-L_0(\lambda)|}{|R^*_0(\lambda)|} \leq \sup_{\lambda\in\Lambda}\frac{|L_n^*(\lambda)-R^*_0(\lambda)|}{|R^*_0(\lambda)|} + \sup_{\lambda\in\Lambda}\frac{|R_0^*(\lambda)-L_0(\lambda)|}{|R^*_0(\lambda)|}. \notag
\end{align}
Therefore, we come to prove (\ref{eq:32}). 

Notice that
\begin{align}
\label{eq:34}
    \sup_{\lambda\in\Lambda} \frac{|R^*_0(\lambda)-L_0(\lambda)|}{|R^*_0(\lambda)|} & \leq \xi_n^{-1} \sup_{\lambda\in\Lambda}|L_0(\lambda)-R^*_0(\lambda)|\notag\\
    & \leq \xi_n^{-1} \sup_{\lambda\in\Lambda}\left\{|L_0(\lambda)-L^*_0(\lambda)|+|L^*_0(\lambda)-R^*_0(\lambda)|\right\}. 
\end{align}
Similar to (\ref{eq:14}) in the proof of Theorem \ref{theorem:1}, we have
\begin{align}
\label{eq:35}
    & \xi_n^{-1}\sup_{\lambda\in\Lambda}|L_0(\lambda)-L^*_0(\lambda)| \notag\\
    =\ & \xi_{n}^{-1}  \sup_{\lambda\in\Lambda}\Big|n_0^{-1}\sum_{j=1}^{n_0}\big[\{f_{\lambda}(X_j^0,\hat{w}_n(\lambda))-E(Y_j^0|X_j^0)\}^2 -\{f_{\lambda}(X^0_j,w^*(\lambda))-E(Y^0_j|X^0_j)\}^2\big]\Big|\notag\\ 
    =\ & O_p(\xi_{n}^{-1}n_0^{-1/2}) =o_p(1),
\end{align}
where the second equality follows from (\ref{eq:13}) and the third equality follows from Assumption \ref{assumption:4}. Then, by the definition that $R_0^*(\lambda)=E(L_0^*(\lambda))$, it holds for $\lambda\in\Lambda$, $\delta>0$ that
\begin{align}
    \Pr\big(|L^*_0(\lambda)-R^*_0(\lambda)|>\delta\big) &\leq \delta^{-2}\operatorname{var}(L^*_0(\lambda))\notag\\ 
    &= n_0^{-1}\delta^{-2}\operatorname{var}\big[\{f_{\lambda}(X_0,\hat{w}_n(\lambda))-E(Y_0|X_0)\}^2\big]
\end{align}
By Assumption \ref{assumption:6}, we obtain 
\begin{align}
\label{eq:37}
    \xi_n^{-1}\sup_{\lambda\in\Lambda}|L^*_0(\lambda)-E\{L^*_0(\lambda)\}| = O_p(\xi_n^{-1}n_0^{-1/2}) = o_p(1).
\end{align}
Combining (\ref{eq:34}), (\ref{eq:35}), and (\ref{eq:37}), we obtain (\ref{eq:32}). Thus we complete the proof of (\ref{eq:7}). 

Now we are ready to prove (\ref{eq:ad1}) and (\ref{eq:ad2}). According to (\ref{eq:7}), we have
\begin{align} 
    \frac{L_0(\hat{\lambda})-\inf_{\lambda \in \Lambda} L_0(\lambda)}{\inf _{\lambda \in \Lambda} L_0(\lambda)} = \frac{L_0(\hat{\lambda})}{\inf _{\lambda \in \Lambda} L_0(\lambda)} -1 =o_p(1)
\end{align}
By the uniformly integrablity in Assumption \ref{assumption:7}(ii), we obtain
\begin{align} 
    \frac{L_0(\hat{\lambda})-\inf_{\lambda \in \Lambda} L_0(\lambda)}{\inf _{\lambda \in \Lambda} L_0(\lambda)}\stackrel{L_1}{\longrightarrow} 0, \notag
\end{align}
which implies
\begin{align}\label{eq:ad5}
    E\Big\{\frac{L_0(\hat{\lambda})}{\inf_{\lambda \in \Lambda} L_0(\lambda)}\Big\} \longrightarrow 1.
\end{align}
This proves (\ref{eq:ad1}). As for (\ref{eq:ad2}), we have
\begin{align}
    & E\big\{L_0(\hat{\lambda})-\inf_{\lambda \in \Lambda} L_0(\lambda)\big\}\notag\\
    \leq\ & E\big\{\frac{L_0(\hat{\lambda})-\inf_{\lambda \in \Lambda} L_0(\lambda)}{\inf_{\lambda \in \Lambda} L_0(\lambda)}\big\}\times  E\big[\inf_{\lambda \in \Lambda} L_0(\lambda) \cdot \{L_0(\hat{\lambda})-\inf_{\lambda \in \Lambda} L_0(\lambda)\}\big] \notag\\
    <\  & C\times E\big\{\frac{L_0(\hat{\lambda})-\inf_{\lambda \in \Lambda} L_0(\lambda)}{\inf_{\lambda \in \Lambda} L_0(\lambda)}\big\}\times E\{\inf_{\lambda \in \Lambda} L_0(\lambda)\} \notag
\end{align}
where the first inequality follows from Cauchy-Schwarz inequality and the second inequality follows from Assumption \ref{assumption:7}(i). This implies
\begin{align}\label{eq:ad7}
    0\leq \frac{E\big\{L_0(\hat{\lambda})-\inf_{\lambda \in \Lambda} L_0(\lambda)\big\}}{E\{\inf_{\lambda \in \Lambda} L_0(\lambda)\}} \leq C \times E\big\{\frac{L_0(\hat{\lambda})-\inf_{\lambda \in \Lambda} L_0(\lambda)}{\inf_{\lambda \in \Lambda} L_0(\lambda)}\big\}.
\end{align}
By (\ref{eq:ad5}), we know that the right hand side of (\ref{eq:ad7}) converges to 0, which also implies
\begin{align}
    \frac{E\big\{L_0(\hat{\lambda})-\inf_{\lambda \in \Lambda} L_0(\lambda)\big\}}{E\{\inf_{\lambda \in \Lambda} L_0(\lambda)\}}  \longrightarrow 0. \notag
\end{align}
Therefore, we completes the proof of Theorem \ref{theorem:2}.
\end{proof}

\section{Design of models in Image classification problems}  \label{appendix:3}
We employ the same structure for model $F_1$ and $F_2$, which follows the standard AlexNet architecture \citep{krizhevsky2012imagenet}, with necessary modifications tailored to our tasks. The model consists of two parts: feature extraction and classification. Specifically, the feature extraction part is composed of the following layers in sequence: convolutional layer 1 (32 filters of $5\times5$), convolutional layer 2 (64 filters of $3\times3$), max-pooling layer 1 (kernel size $2\times2$ and stride 2), convolutional layer 3 (96 filters of $3\times3$, convolutional layer 4 (64 filters of $3\times3$, convolutional layer 5 (32 filters of $3\times3$, max-pooling layer 1 (kernel size $2\times2$ and stride 1). The classification part comprises three fully connected layers with two ReLu acivation functions, having output sizes of 512, 128, and 1. Dropout layers are placed between the fully connected layers. An input image will be first normalized to have values in $[-1,1]$ and then processed through the feature extraction and classification components, and a sigmoid activation function is applied at the end of the network to generate the final output probabilities.

We train models $F_1$ and $F_2$ following the procedure outlined in Section \ref{sec:4.2}, using 60000 relabeled samples from the training set of the Fashion-MNIST or MNIST datasets, respectively. The achieved accuracies of $F_1$ and $F_2$ on their repective 10000 test samples exceeds 99.5\%.

\end{appendix}

\vskip 0.2in

\bibliography{reference2}

\end{document}